%% file: arxiv-version.tex
\documentclass{article}

\usepackage{microtype}
\usepackage{graphicx}
\usepackage{subcaption}
\usepackage{booktabs} % for professional tables

\usepackage{hyperref}

\usepackage[preprint]{icml2026}

\usepackage{amsmath}
\usepackage{amssymb}
\usepackage{mathtools}
\usepackage{amsthm}
\usepackage{enumitem}
\usepackage[most]{tcolorbox}
\allowdisplaybreaks
\usepackage[capitalize,noabbrev]{cleveref}

\theoremstyle{plain}
\newtheorem{theorem}{Theorem}[section]

\newtheorem{lemma}[theorem]{Lemma}
\newtheorem{corollary}[theorem]{Corollary}
\theoremstyle{definition}
\newtheorem{definition}[theorem]{Definition}
\newtheorem{assumption}[theorem]{Assumption}
\theoremstyle{remark}

\DeclareMathOperator*{\argmax}{arg\,max}
\DeclareMathOperator*{\argmin}{arg\,min}

\usepackage[textsize=tiny]{todonotes}

\icmltitlerunning{Provably Efficient Algorithms for S- and Non-Rectangular RMDPs}

\begin{document}

\twocolumn[
  \icmltitle{Provably Efficient Algorithms for S- and Non-Rectangular Robust MDPs with General Parameterization}



  \icmlsetsymbol{equal}{*}

  \begin{icmlauthorlist}
    \icmlauthor{Anirudh Satheesh}{umd}
    \icmlauthor{Ziyi Chen}{umd}
    \icmlauthor{Furong Huang}{umd}
    \icmlauthor{Heng Huang}{umd}
  \end{icmlauthorlist}

  \icmlaffiliation{umd}{Department of Computer Science, University of Maryland, College Park}

  \icmlcorrespondingauthor{Anirudh Satheesh}{anirudhs@terpmail.umd.edu}

  \icmlkeywords{Machine Learning, ICML}

  \vskip 0.3in
]



\printAffiliationsAndNotice{}  % no special notice (required even if empty)

\begin{abstract}
    We study robust Markov decision processes (RMDPs) with general policy parameterization under s-rectangular and non-rectangular uncertainty sets. Prior work is largely limited to tabular policies, and hence either lacks sample complexity guarantees or incurs high computational cost. Our method reduces the average reward RMDPs to entropy-regularized discounted robust MDPs, restoring strong duality and enabling tractable equilibrium computation. We prove novel Lipschitz and Lipschitz-smoothness properties for general policy parameterizations that extends to infinite state spaces. To address infinite-horizon gradient estimation, we introduce a multilevel Monte Carlo gradient estimator with $\tilde{\mathcal{O}}(\epsilon^{-2})$ sample complexity, a factor of $\mathcal{O}(\epsilon^{-2})$ improvement over prior work. Building on this, we design a projected gradient descent algorithm for s-rectangular uncertainty ($\mathcal{O}(\epsilon^{-5})$) and a Frank--Wolfe algorithm for non-rectangular uncertainty ($\mathcal{O}(\epsilon^{-4})$ discounted, $\mathcal{O}(\epsilon^{-10.5})$ average reward), significantly improving prior results in both the discounted setting and average reward setting. Our work is the first one to provide sample complexity guarantees for RMDPs with general policy parameterization beyond $(s, a)$-rectangularity. It also provides the first such guarantees in the average reward setting and improves existing bounds for discounted robust MDPs.
\end{abstract}
\input{Sections/s1_introduction}
\input{Sections/s2_related_work}
\input{Sections/s3_formulation}
\input{Sections/s4_transition_kernel_optimization}
\input{Sections/s5_algorithms}
\input{Sections/s6_stochastic_gradient_estimation}
\input{Sections/s7_theoretical_analysis}
\input{Sections/s8_conclusion}
\section*{Acknowledgments}
Ziyi Chen and Heng Huang were partially supported by NSF IIS 2347592, 2348169, DBI 2405416, CCF 2348306, CNS 2347617, RISE 2536663. Satheesh and Huang F. are supported by DARPA Transfer from Imprecise and Abstract Models to Autonomous Technologies (TIAMAT) 80321, DARPA HR001124S0029-AIQ-FP-019, DOD-AFOSR-Air Force Office of Scientific Research under award number FA9550-23-1-0048, National Science Foundation TRAILS Institute (2229885). Private support was provided by Peraton and Open Philanthropy. The Authors acknowledge the National Artificial Intelligence Research Resource (NAIRR) Pilot for contributing to this research result.

% \textbf{Do not} include acknowledgements in the initial version of the paper

% If a paper is accepted, the final camera-ready version can (and usually should)
% include acknowledgements.  Such acknowledgements should be placed at the end of
% the section, in an unnumbered section that does not count towards the paper
% page limit. Typically, this will include thanks to reviewers who gave useful
% comments, to colleagues who contributed to the ideas, and to funding agencies
% and corporate sponsors that provided financial support.

\section*{Impact Statement}
This paper presents work whose goal is to advance the field of Machine
Learning. There are many potential societal consequences of our work, none
which we feel must be specifically highlighted here.


\bibliography{example_paper}
\bibliographystyle{icml2026}

%%%%%%%%%%%%%%%%%%%%%%%%%%%%%%%%%%%%%%%%%%%%%%%%%%%%%%%%%%%%%%%%%%%%%%%%%%%%%%%
%%%%%%%%%%%%%%%%%%%%%%%%%%%%%%%%%%%%%%%%%%%%%%%%%%%%%%%%%%%%%%%%%%%%%%%%%%%%%%%
% APPENDIX
%%%%%%%%%%%%%%%%%%%%%%%%%%%%%%%%%%%%%%%%%%%%%%%%%%%%%%%%%%%%%%%%%%%%%%%%%%%%%%%
%%%%%%%%%%%%%%%%%%%%%%%%%%%%%%%%%%%%%%%%%%%%%%%%%%%%%%%%%%%%%%%%%%%%%%%%%%%%%%%
\newpage
\appendix
\onecolumn
\input{Sections/Appendix/related_work}
\input{Sections/Appendix/properties}

\input{Sections/Appendix/gradient_dominance}
\input{Sections/Appendix/algorithms}
\input{Sections/Appendix/stochastic_gradient_estimation}
\input{Sections/Appendix/theoretical_analysis}
\input{Sections/Appendix/supporting_lemmas}

\input{Sections/Appendix/limitations_and_future_work}
\end{document}

%% file: Sections/s1_introduction.tex
\section{Introduction}
Sequential decision-making under model uncertainty is fundamental for deploying reinforcement learning systems in the real world. Agents are frequently trained using imperfect simulators and must commit to policies before interacting with unknown or evolving environments. In domains such as robotics, even minor mismatches between simulated dynamics and real-world conditions can result in severe performance degradation or catastrophic failure \citep{peng2018sim}. To address this challenge, Robust Markov Decision Processes (RMDPs) provide a principled framework by optimizing policies against worst-case transition dynamics drawn from a predefined uncertainty set, yielding explicit robustness guarantees \citep{iyengar2005robust, nilim2003robustness}.
A critical limitation of current RMDP approaches is their reliance on tabular representations, which become intractable in high-dimensional or continuous environments \citep{goyal2023robust, li2023policy, pmlr-v235-chen24s}. Additionally, existing methods typically minimize over transition kernels before optimizing the policy, requiring projection of the resulting policy back onto the feasible space, which is computationally infeasible for more general policy parameterizations.
% Research Question 1: How can we design provably efficient algorithms for Robust MDPs with general policy parameterizations?
\begin{tcolorbox}[
    colback=white,
    colframe=black,
    arc=3mm,
    boxrule=1.5pt,
    left=6pt,
    right=6pt,
    top=6pt,
    bottom=6pt
]
\textit{RQ1: How can we design provably efficient algorithms for Robust MDPs with general policy parameterizations?}
\end{tcolorbox}

Furthermore, while the discounted setting has received significant attention, many tasks require guarantees on long-term performance, motivating robustness in the average-reward setting \citep{wang2023robust}. However, this transition introduces severe theoretical hurdles. Unlike their discounted counterparts, average-reward problems often lack contraction properties of the Bellman operator, and the strong duality typically required to exchange minimization and maximization operators fails \citep{grand2023beyond}. Consequently, prior work lacks end-to-end sample complexity guarantees for this setting \citep{wang2025provable}, leaving a significant gap in providing robustness for long-horizon tasks.

% Research Question 2: How can we overcome the lack of strong duality to provide the first end-to-end sample complexity guarantees in the average-reward setting?
\begin{tcolorbox}[
    colback=white,
    colframe=black,
    arc=3mm,
    boxrule=1.5pt,
    left=6pt,
    right=6pt,
    top=6pt,
    bottom=6pt
]
\textit{RQ2: How can we overcome the lack of strong duality to provide the first end-to-end sample complexity guarantees in the average-reward setting?}
\end{tcolorbox}

\subsection{Challenges and Contributions}
To mitigate the issue of projecting back onto the feasible policy class \(\Pi\), we reverse the optimization order and update the policy directly, enabling gradient-based methods over expressive policy classes. Although this approach is effective in the discounted setting, it does not directly extend to average-reward problems, where strong duality fails \citep{grand2023beyond, wang2025bellman}. To overcome this challenge, we introduce an entropy-regularized discounted reduction that transforms the average-reward RMDP into a discounted robust problem, restoring the minimax structure required for optimization.
In the discounted setting, we develop novel Lipschitz and smoothness bounds for entropy-regularized robust value functions that hold under both general policy and linear transition parameterizations. Unlike previous work, our bounds are independent of state space size, enabling scalability to continuous or infinite large spaces. To reduce the bottleneck from infinite-horizon gradient estimation, we introduce a multilevel Monte Carlo (MLMC) gradient estimator, achieving a \(\mathcal{O}(\epsilon^{-2})\) improvement over temporal-difference methods used in \citet{pmlr-v235-chen24s}. We incorporate these methods into a projected gradient descent algorithm for $s$-rectangular uncertainty sets, yielding a complexity of \(\mathcal{O}(\epsilon^{-5})\). We also develop a Frank-Wolfe algorithm for non-rectangular sets, yielding complexities \(\mathcal{O}\left(C_f^2\epsilon^{-4}\right)\) in the discounted setting and \(\mathcal{O}(C_f^2\epsilon^{-10.5}H^{5.5})\) for average-reward problems, up to an irreducible error, where \(H\) is the span of the optimal policy. This provides the first sample complexity guarantees for average-reward RMDPs with general parameterization, while improving efficiency in the discounted case.

We summarize our contributions below:
\begin{itemize}[noitemsep,nolistsep,topsep=1pt]
    \item \textbf{First results for average-reward RMDPs.} We provide the first end-to-end sample complexity guarantees for both $s$-rectangular and non-rectangular uncertainty sets in the average-reward setting.
    \item \textbf{Support for general parameterizations and infinite state spaces.} By reversing the optimization order and establishing bounds independent of state space size, our framework enables the first scalable algorithms under general policy parameterization and infinite state spaces.
    \item \textbf{Novel MLMC Gradient Estimator for improved efficiency.} Our estimator reduces the gradient estimation bottleneck achieving a \(\mathcal{O}(\epsilon^{-2})\) improvement over prior work.
\end{itemize}

\begin{table*}[ht]
\centering
\label{tab:rmdp-comparison}
\resizebox{0.75\textwidth}{!}{%
\begin{tabular}{@{}lllllll@{}}
\toprule
\textbf{Reference} & \textbf{Parameterization} & \textbf{Uncertainty} & \textbf{Reward} & \textbf{Setting} & \textbf{Iteration Complexity} & \textbf{Sample Complexity}\\ \midrule
\citet{li2023policy} & Tabular & Non-rect. & Discounted & Stoch. & $O(\epsilon^{-4})$ & N/A \\
\citet{wang2025provable} & Tabular & Non-rect. & Average & Det. & $O(\epsilon^{-4})$ & N/A \\
\citet{ghosh2025scaling} & General & $(s, a)$-rect. & Discounted & Det. & $O(\epsilon^{-2})$ & $O(\epsilon^{-2})$ \\
\citet{pmlr-v235-chen24s} & Tabular & $s$-rect. & Discounted & Det. & $O(\epsilon^{-3})$ & $O(\epsilon^{-3})$ \\
\citet{pmlr-v235-chen24s} & Linear & $s$-rect. & Discounted & Stoch. & $O(\epsilon^{-3})$ & $O(\epsilon^{-7})$ \\ \midrule
\textbf{This Work (PGD)} & \textbf{General} & \textbf{$s$-rect.} & \textbf{Discounted} & \textbf{Stoch.} & \boldmath$O(\epsilon^{-3})$ & \boldmath$O(\epsilon^{-5})$ \\
\textbf{This Work (FW)} & \textbf{General} & \textbf{Non-rect.} & \textbf{Discounted} & \textbf{Stoch.} & \boldmath$O(C_f^2 \epsilon^{-2})$ & \boldmath$O(C_f^2 \epsilon^{-4})$ \\
\textbf{This Work (FW)} & \textbf{General} & \textbf{Non-rect.} & \textbf{Average} & \textbf{Stoch.} & \boldmath$O(C_f^2 \epsilon^{-2})$ & \boldmath$O(C_f^2 \epsilon^{-10.5}H^{5.5})$ \\ \bottomrule
\end{tabular}%
}
\caption{\textbf{Comparison of Theoretical Guarantees for Robust MDPs.} Our work is the first to achieve sample complexity guarantees in the average reward setting and the first to develop algorithms that work in the general policy parameterization setting.}
\end{table*}

%% file: Sections/s2_related_work.tex
\section{Related Work}
\label{sec:related_work}

\paragraph{Parameterization of Robust MDPs.} 
Traditional RMDP literature focuses heavily on $(s, a)$-rectangular uncertainty sets, which allow for polynomial-time planning but are restricted to tabular or linear settings \citep{iyengar2005robust, nilim2003robustness}. While $s$-rectangular and non-rectangular sets offer more realistic modeling of coupled dynamics, they have remained largely intractable beyond the tabular case \citep{wiesemann2013robust, goyal2023robust}. The most closely related work, \citet{ghosh2025scaling}, introduces general function approximation but is limited to $(s, a)$-rectangularity and discounted settings. In contrast, our framework is the first to support general parameterization for more complex $s$-rectangular and non-rectangular uncertainty sets in both discounted and average-reward regimes.

\paragraph{Average-Reward Duality Gap.} 
Transitioning from discounted to average-reward RMDPs introduces severe theoretical hurdles. Specifically, strong duality typically fails, meaning the minimax structure required to exchange policy optimization and worst-case kernel estimation is lost \citep{grand2023beyond, wang2025bellman}. The closest work for robust average reward MDPs beyond $(s, a)$-rectangularity is \citet{wang2025provable}, which is limited to tabular settings and lacks sample complexity guarantees. We overcome this by introducing an entropy-regularized discounted reduction that restores strong duality, enabling the first end-to-end sample complexity analysis in the average reward setting.

%% file: Sections/s3_formulation.tex
\section{Formulation}
\subsection{Robust Average Reward MDPs}
We first consider an average-reward infinite-horizon MDP represented as the tuple 
\((\mathcal{S}, \mathcal{A}, \mathcal{P}, r, \rho, P^\circ)\), where \(\mathcal{S}\) is the state space, \(\mathcal{A}\) is the action space with cardinality \(A\), \(r: \mathcal{S} \times \mathcal{A} \to [0,1]\) is the reward function, \(P^\circ: \mathcal{S} \times \mathcal{A} \to \Delta(\mathcal{S})\) is the nominal transition kernel, and \(\rho\) is the initial state distribution. We define \(\mathcal{P}\) as an uncertainty set of transition kernels that contains the nominal transition kernel. There are three main classes of uncertainty sets:
\begin{definition}[Typical Uncertainty Sets]
Uncertainty sets can be broadly classified into one of three groups:
    \begin{enumerate}[noitemsep,nolistsep,topsep=1pt]
        \item \((s,a)\)-rectangular \citep{nilim2003robustness, iyengar2005robust}: \(\mathcal{P} = \left\{\, p \in (\Delta^{\mathcal{S}})^{\mathcal{S} \times \mathcal{A}} : p(\cdot \mid s, a) \in \mathcal{P}_{s,a}, \ \forall s \in \mathcal{S}, a \in \mathcal{A} \right\}.\) where each \(\mathcal{P}_{s, a}\subseteq\Delta(\mathcal{S})\).
        \item \(s\)-rectangular \citep{wiesemann2013robust, pmlr-v235-chen24s}: \(\mathcal{P} = \left\{\, p \in (\Delta^{\mathcal{S}})^{\mathcal{S} \times \mathcal{A}} : p(\cdot \mid s, \cdot) \in \mathcal{P}_{s}, \ \forall s \in \mathcal{S}, \right\}\) where each \(\mathcal{P}_s\subseteq\Delta(\mathcal{S})^{\mathcal{A}}\) .
        \item Non-rectangular \citep{goyal2023robust, li2023policy, wang2025provable}: \(\mathcal{P}\) cannot be decomposed as a Cartesian product over states.
    \end{enumerate}
\end{definition}
In this paper, we consider the more general \(s\)-rectangular and non-rectangular cases where \(\mathcal{P}\) is convex and compact. For a fixed policy \(\pi \in \Pi\) and transition kernel \(P\), we define the long term average reward as
\begin{align}
    g_{P}^{\pi}(s) = \lim_{T \to \infty} \mathbb{E}_{\pi, P}\left[\frac{1}{T} \sum_{t=0}^{T-1} r_t \bigg| s_0 = s\right] 
\end{align}
where \(r_t = r(s_t, a_t)\) and the robust average reward is \(g_{\mathcal{P}}^\pi(s) = \min_{P \in \mathcal{P}} g_{P}^\pi\). Additionally, we consider a parameterized class of policies \(\Pi\), which consists of all policies parameterized by \(\theta \in \Theta \subset \mathbb{R}^d\), where \(d \ll |\mathcal{S}||\mathcal{A}|\). We also assume the optimal parameter is unique up to equivalence classes \citep{GIVAN2003163, gopalan2025dpo}. This assumption ensures is essential for Danskin's theorem (Lemma~\ref{lemma: outer_smoothness}) and thus enabling our algorithms\footnote{To the best of our knowledge, this is the weakest assumption that allows gradient based optimization. This holds for many parameterizations, such as softmax NN and energy-based models.}.
\begin{assumption}[Unique Optimality up to Equivalence Classes]
    If $\theta^*_1$ and $\theta^*_2$ are both optimal, then $\pi_{\theta^*_1} = \pi_{\theta^*_2}$.
\end{assumption}  
We note that for any policy \(\pi_\theta\) and transition kernel \(P\), the sequence of states from the MDP is a Markov Chain. Thus, we impose a standard ergodicity assumption.
\begin{assumption}[Ergodicity]
\label{assumption: ergodicity}
    For all $\theta \in \Theta$ and all $P \in \mathcal{P}$, the Markov chain induced by $(\pi_\theta,P)$ is irreducible and aperiodic.
\end{assumption}
Assumption \ref{assumption: ergodicity} guarantees that each policy–kernel pair \((\pi_\theta,P)\) admits a unique stationary distribution \(d_{\pi_\theta}^P\) satisfying \(d_{\pi_\theta}^P = (P_{\pi_\theta})^\top d_{\pi_\theta}^P\), where \(P_{\pi_\theta}\) is the induced transition kernel by the policy \(\pi_\theta\) and is commonly used in prior work \citep{gong2020duality, pesquerel2022regret, ganesh2025orderneural, ganesh2025a, ganesh2025order}, even in the robust setting \citep{wang2025provable}. Our objective is to find, for a predefined tolerance \(\epsilon\), a policy parameterized by \(\theta \in \Theta\) such that
\begin{align}
    \min_{P \in \mathcal{P}} g_{P}^{\pi_{\theta}} \geq \max_{\theta' \in \Theta}  \min_{P \in \mathcal{P}} g_{P}^{\pi_{\theta'}} - \epsilon
\end{align}
\subsection{On Duality and Discounted Reductions}
Strong duality can fail in robust average-reward MDPs, even when the uncertainty set $\mathcal{P}$ is convex and compact. In particular, the max–min problem
\[
\max_{\pi \in \Pi} \min_{P \in \mathcal{P}} g_P^\pi
\]
does not generally admit a saddle point, and exchanging $\max$ and $\min$ can be invalid in the average-reward setting \citep{grand2023beyond, wang2025bellman}. As a result, we opt to reduce our average-reward problem to the discounted, where contraction of the Bellman operator restores strong duality. In particular, \citet{wang2022near} show that for an average-reward MDP with optimal policy span $H$, one can choose a discount factor $\gamma = 1 - \Theta(\epsilon/H)$, solve the induced discounted MDP to accuracy $\epsilon_\gamma = O(\epsilon/(1-\gamma))$, and return the resulting policy to obtain an $\epsilon$-optimal policy for the original average-reward problem.

\subsection{Entropy-Regularized Robust MDPs}
Policy learning in robust MDPs can suffer from instability because worst-case transitions tend to push the optimization toward sharp or brittle policies. To counter this effect, we introduce an entropy regularization term that encourages smoother policies and promotes adequate exploration. The entropy-regularized discounted return under kernel \(P\) with discount factor \(\gamma \in (0,1)\) is
\begin{align}
    J_{P, \tau}^{\pi_\theta}(s) = \mathbb{E}_{\pi_\theta, P} \left[\sum_{t = 0}^{\infty} \gamma^t r_{t, \tau} \,\Big|\, s_0 = s \right]
\end{align}
where \(r_{t, \tau} = r_t - \tau \log \pi_{\theta}(s_t, a_t)\) and \(\tau > 0\) controls the strength of regularization. We define the entropy-regularized state and action-value functions
\begin{align}
    V_{P, \tau}^{\pi_\theta}(s) &= \mathbb{E}_{\pi_\theta,P} \left[ \sum_{t = 0}^{\infty} \gamma^t r_{t, \tau} \,\Big|\, s_0 = s \right], \label{eqn: discounted value function}\\
    Q_{P, \tau}^{\pi_\theta}(s, a) &= \mathbb{E}_{\pi_\theta,P} \left[ \sum_{t = 0}^{\infty} \gamma^t r_{t, \tau} \,\Big|\, s_0 = s, a_0 = a \right]. \label{discounted q function}
\end{align}
The entropy-regularized Bellman equation becomes
\begin{align}
    \label{eqn: discounted entropy bellman}
    Q_{P, \tau}^{\pi_\theta}(s,a) = r_{\tau}(s,a) + \gamma \, \mathbb{E}_{s' \sim P} \left[ V_{P, \tau}^{\pi_\theta}(s') \right].
\end{align}

\subsection{Extensions to Large and/or Infinite State Spaces}

To enable tractable robust optimization in large or infinite state spaces, we adopt a linear parameterization of the transition kernel. Specifically, suppose there exists a feature mapping \(\phi: \mathcal{S} \times \mathcal{A} \times \mathcal{S} \to \mathbb{R}^{d_p}\) and parameters \(\xi \in \mathbb{R}^{d_p}\) such that the transition kernel can be expressed as
\begin{align}
    P_{\xi}(s'\mid s,a)=\langle \phi(s,a,s'), \xi \rangle,
\end{align}
where for all \((s,a)\) the feature vector \(\phi(s,a,\cdot)\) lies in the probability simplex, and the inner product \(\langle \phi(s,a,s'), \xi \rangle\) defines a valid probability distribution over next states. We assume \(\|\phi(s, a, \cdot)\|_1 \leq C_\phi\). Equivalently, this can be written as \(P_{\xi} = \Phi \xi\), with \(\Phi \in \mathbb{R}^{|\mathcal{S}|^2|\mathcal{A}| \times d_p}\), where each row of \(\Phi\) is \(\phi^\top(s, a, s')\). This form arises naturally in the linear mixture MDP literature, where transition dynamics are expressed as weighted combinations of basis distributions with unknown weights \(\xi\) \citep{liu2025linear}.  

While an exact linear representation of every transition kernel is unrealistic, we can relate this linear class to the original uncertainty set \(\mathcal{P}\) by measuring approximation quality in terms of the Wasserstein‑1 metric \(W_1\), which quantifies the distance between probability distributions under a ground metric on \(\mathcal{S}\). We assume the following:

\begin{assumption}
    \label{assumption: transition kernel model error}
    For the uncertainty set in the parameter space \(\Xi \subseteq \mathbb{R}^{d_p}\) and a desired model error tolerance \(\epsilon_{\mathrm{model}}\), there exists a feature mapping \(\phi\) such that
    \begin{align}
        \sup_{P \in \mathcal{P}} \inf_{\xi \in \Xi} \sup_{s, a} W_{1}\left(P(\cdot\mid s,a),\,P_{\xi}(\cdot\mid s,a)\right) \leq \epsilon_{\mathrm{model}}.
    \end{align}
\end{assumption}

Under this assumption, every kernel in \(\mathcal{P}\) can be approximated up to \(\epsilon_{\mathrm{model}}\) by a linearly parameterized kernel \(P_{\xi}\) with \(\xi \in \Xi\). Therefore, robust optimization can be carried out directly over the lower-dimensional parameterized uncertainty set \(\Xi\) rather than over the full kernel space. We note that this assumption is much weaker than model error assumptions that use the TV distance and more naturally extends to infinite and continuous state spaces \citep{asadi2018lipschitz}. Consequently, we can define an \((\epsilon, \tau)\)-Nash equilibrium in the parameterized uncertainty set, similar to \citet{pmlr-v235-chen24s}, as follows:
\begin{definition}
A pair \((\theta, \xi) \in \Theta \times \Xi\) is an \((\epsilon, \tau)\)-Nash equilibrium if
\begin{align}
    J_{P_\xi,\tau}^{\pi_\theta} - \max_{\theta' \in \Theta} J_{P_\xi,\tau}^{\pi_{\theta'}} &\leq \epsilon, \\
    \min_{\xi' \in \Xi} J_{P_{\xi'},\tau}^{\pi_\theta} - J_{P_\xi,\tau}^{\pi_\theta} &\leq \epsilon.
\end{align}
\end{definition}

\begin{lemma}
    \label{lemma: entropy regularized policy optimality}
    Under Assumption~\ref{assumption: transition kernel model error} and if the value function \(V_{P_\xi, \tau}^{\pi_\theta}\) is \(L_V\)-Lipschitz uniformly with respect to the state distance metric, for any \(\epsilon \geq 0\) and \(\tau > 0\) an \((\epsilon, \tau)\)-Nash Equilibrium exists. If \((\theta, \xi)\) is such an \((\epsilon, \tau)\)-Nash Equilibrium, then \(\pi_\theta\) is an \((2\epsilon + \frac{\tau \log |\mathcal{A}|}{1-\gamma} + \frac{2L_V \epsilon_{\mathrm{model}}}{1-\gamma})\)-optimal policy to the original optimization problem.
\end{lemma}
We present the proof in Appendix~\ref{appendix: properties of entropy regularized DRMDPs}. Lemma~\ref{lemma: entropy regularized policy optimality} allows us to set \(\tau = O(\epsilon(1-\gamma))\) to achieve global optimality, while allowing us to develop efficient sample complexity guarantees using the properties of entropy regularized MDPs.

%% file: Sections/s4_transition_kernel_optimization.tex
\section{Optimization of Transition Kernel Parameters}
In robust Markov Decision Processes (MDPs), the optimization of transition kernel parameters \(\xi\) plays a central role in defining worst-case dynamics and guiding policy improvement. In this section, we formalize the theoretical foundations enabling gradient-based optimization over the transition parameters, characterize the resulting geometric and regularity properties, and demonstrate how these properties facilitate the design of efficient algorithms for robust average reward MDPs.

\subsection{Duality and \(s\)-Rectangularity}
Our core problem can be formulated as a maximin optimization:
\begin{equation}
    \max_{\theta \in \Theta} \min_{\xi \in \Xi} J_{P_\xi, \tau}^{\pi_\theta},
\end{equation}
where \(J_{P_\xi, \tau}^{\pi_\theta}\) denotes the entropy-regularized average reward under policy $\pi_\theta$ and transition kernel \(P_\xi\). Directly solving this problem is challenging due to the potential non-concavity of \(g\) with respect to the transition kernel \(\xi\). 

A key structural property that enables tractability is \emph{\(s\)-rectangularity}. If the uncertainty set \(\mathcal{P}\) is \(s\)-rectangular, then each state’s transition uncertainty is independent of other states, yielding a convex and compact uncertainty set for each state’s conditional distribution. This structural decomposition implies that the robust Bellman operator acts independently over states, allowing for strong duality to hold. In particular, following the results in \citet{mai2021robust} for discounted entropy-regularized robust MDPs with \(s\)-rectangular uncertainty:
\begin{equation}
    \max_{\theta \in \Theta} \min_{\xi \in \Xi} J_{P_\xi, \tau}^{\pi_\theta} = \min_{\xi \in \Xi} \max_{\theta \in \Theta} J_{P_\xi, \tau}^{\pi_\theta}.
\end{equation}
This duality is critical as we can first solve for the optimal policy given a fixed transition kernel, then optimize the transition kernel using the resulting policy. We define the dual objective as
\begin{equation}
    F(\xi) \coloneqq \max_{\theta \in \Theta} J_{P_\xi, \tau}^{\pi_\theta}, \quad \xi^* = \arg\min_{\xi \in \Xi} F(\xi)
\end{equation}
representing the worst-case transition kernel within the parameterized uncertainty set.

\subsection{Gradient Dominance under Linear Parameterization}
To ensure efficient convergence of gradient-based methods over $\xi$, it is important to understand the geometry of $F(\xi)$. While $F$ is generally non-convex in $\xi$, it does hold a gradient dominance property that guarantees that any sub-optimality in the objective can be upper-bounded by the directional derivative along the gradient, which allows global convergence rates to be established even without strong convexity.

\begin{lemma}[Gradient Dominance under Linear Parameterization]
\label{lemma: gradient dominance linear transition kernel}
Let $P_\xi(s' \mid s,a) = \phi(s,a,s')^\top \xi$ for $\xi \in \Xi$. For any fixed policy $\pi_\theta$, the entropy-regularized objective $J_{P_\xi, \tau}^{\pi_\theta}$ satisfies the gradient dominance property:
\[
    \min_{\xi' \in \Xi} \left( J_{P_\xi, \tau}^{\pi_\theta} - J_{P_{\xi'}, \tau}^{\pi_\theta} \right)
    \leq \frac{D}{1 - \gamma} \max_{\xi' \in \Xi} (\xi - \xi')^\top \nabla_\xi J_{P_\xi, \tau}^{\pi_\theta},
\]
where \(D\) denotes the distribution mismatch coefficient \citep{agarwal2021theory, leonardos2021global, wang2025provable}:
\[
    D \coloneqq \max_{\xi, \xi' \in \Xi} \left\| \frac{d_{\pi_\theta}^{P_{\xi'}}}{d_{\pi_\theta}^{P_\xi}} \right\|_\infty
\]
\end{lemma}
We present the full proof in Appendix~\ref{appendix: gradient dominance}. This result ensures that, for a fixed policy, the worst-case gain achievable by transitioning to any other kernel within the linear parameterized set is upper-bounded by the gradient projected in the direction of the global optimum. As a consequence, gradient descent over $\xi$ is guaranteed to make consistent progress toward the worst-case kernel.

\subsection{Regularity Properties of the Entropy-Regularized Objective}
The regularity of the entropy-regularized average-reward objective \( J_{P_\xi, \tau}^{\pi_\theta} \)
with respect to both the policy parameters and the transition kernel parameters is central to establishing valid step sizes and accelerated convergence guarantees.

\paragraph{Regularity of the Policy Class.}
We impose standard smoothness and fisher non-degeneracy assumptions on the policy parameterization, which are widely used in the analysis of policy-gradient-based methods \citep{pmlr-v80-papini18a, liu2020improved, agarwal2021theory, xu2019sample, fatkhullin2023stochastic, ganesh2025a}.

\begin{assumption}[Fisher Non-Degeneracy]
    \label{assumption: fisher non degeneracy}
    There exists a constant \(\mu > 0\) such that \(F(\theta) - \mu I_{d}\) is positive semidefinite, where \(F(\theta)\) is the Fisher Information matrix.
\end{assumption}

\begin{assumption}[Smoothness of the Log-Policy]
\label{assumption: log policy smoothness}
For all \(\theta, \theta_1, \theta_2 \in \Theta\) and all \((s,a) \in \mathcal{S} \times \mathcal{A}\), the following hold:
\begin{align*}
    \| \nabla_\theta \log \pi_\theta(a \mid s) \| &\leq G_1, \\ 
    \| \nabla_\theta \log \pi_{\theta_1}(a \mid s) - \nabla_\theta \log \pi_{\theta_2}(a \mid s) \| &\leq G_2 \|\theta_1 - \theta_2\|_2
\end{align*}
\end{assumption}
Assumption~\ref{assumption: fisher non degeneracy} ensures that the optimization problem to find the policy gradient direction is strongly convex. This is required for finding globally optimal policies for a fixed transition kernel. Assumption~\ref{assumption: log policy smoothness} implies that the log-policy is uniformly Lipschitz and Lipschitz smooth in the policy parameters.

\paragraph{Regularity of the Objective.}
Under Assumptions~\ref{assumption: log policy smoothness} and~\ref{assumption: transition kernel model error}, the entropy-regularized objective inherits Lipschitz continuity and smoothness in both arguments.
\begin{lemma}
    \label{lemma: lipschitz and lipschitz smoothness}
    If Assumption~\ref{assumption: log policy smoothness} holds, the entropy-regularized expected discounted cumulative reward \(J_{P_\xi, \tau}^{\pi_\theta}\) and its gradient with respect to the transition parameters \(\nabla_\xi J_{P_\xi, \tau}^{\pi_\theta}\) satisfy the following Lipschitz properties for any \(\theta_1, \theta_2 \in \Pi\) and \(\xi_1, \xi_2 \in \Xi\):  
    \begin{align}
        \left| J_{P_\xi, \tau}^{\pi_{\theta_1}} - J_{P_\xi, \tau}^{\pi_{\theta_2}} \right| &\leq L_\pi  \| \theta_1 - \theta_2 \|_2, \label{eqn: lipschitz theta}  \\
        \left| J_{P_{\xi_1}, \tau}^{\pi_{\theta}} - J_{P_{\xi_2}, \tau}^{\pi_{\theta}} \right| &\leq L_\xi  \| \xi_1 - \xi_2 \|_2, \label{eqn: lipschitz xi}\\
        \left\| \nabla_\xi J_{P_\xi, \tau}^{\pi_{\theta_1}} - \nabla_\xi J_{P_\xi, \tau}^{\pi_{\theta_2}} \right\|_2 &\leq l_{\pi} \| \theta_1 - \theta_2 \|_2, \label{eqn: lipschitz smooth theta} \\
        \left\| \nabla_\xi J_{P_{\xi_1}, \tau}^{\pi_{\theta}} - \nabla_\xi J_{P_{\xi_1}, \tau}^{\pi_{\theta}} \right\|_2 &\leq l_{\xi} \| \xi_1 - \xi_2 \|_2 , \label{eqn: lipschitz smooth xi}
    \end{align}
    where \(L_\pi = \frac{2(1 + \tau \log |\mathcal{A}|)}{(1 - \gamma)^2} G_1 \sqrt{|\mathcal{A}|}, L_\xi = \frac{\sqrt{d_p}C_\phi(1+\log |\mathcal{A}|)}{(1-\gamma)^2}, l_\pi = \frac{5 \gamma (1+ \tau \log |\mathcal{A}|)}{(1-\gamma)^3} C_\phi  G_1 \sqrt{|\mathcal{A}|}, l_\xi = \frac{3\gamma}{(1 - \gamma)^3} \sqrt{d_p}(1+\log |\mathcal{A}|) C_\phi\).
\end{lemma}
These regularity properties ensure that the entropy-regularized objective admits well-controlled gradients with respect to both the policy and transition parameters.

\subsection{Differentiability via Danskin’s Theorem}

The optimization of the transition kernel parameters is formulated as the minimization of the outer objective
\(
F(\xi) := \max_{\theta \in \Theta} J^{\pi_\theta}_{P_\xi,\tau}.
\)
Although \(F\) is defined implicitly through an inner maximization over policies, we show that it admits a simple and tractable gradient characterization. For a fixed transition kernel \(P_\xi\), let \(\pi_{\theta^*}\) denote the optimal policy,
\[
\theta^* \in \arg\max_{\theta \in \Theta} J^{\pi_\theta}_{P_\xi,\tau}.
\]
The entropy regularization with temperature \(\tau > 0\) renders the objective strictly concave in the policy for fixed \(\xi\), ensuring uniqueness of \(\theta^*\) up to equivalence classes. Moreover, \(J^{\pi_\theta}_{P_\xi,\tau}\) is continuously differentiable in \(\xi\) for all \(\theta \in \Theta\), and \(\Theta\) is compact. These properties satisfy the conditions of Danskin’s Theorem. As a result, the outer objective \(F\) is differentiable, and its gradient is given by
\begin{align}
\nabla_\xi F(\xi) = \nabla_\xi J^{\pi_{\theta^*}}_{P_\xi,\tau}.
\end{align}
Although the optimizer \(\pi_{\theta^*}\) depends on \(\xi\), Danskin’s Theorem implies that this dependence does not contribute to the gradient. Equivalently, the gradient admits a stop-gradient interpretation,
\begin{align}
    \nabla_\xi F(\xi) = \nabla_\xi J^{\pi_\theta}_{P_\xi,\tau} \Big|_{\theta=\theta^*}\!, \quad \nabla_\xi \pi_{\theta^*} = 0.
\end{align}
We can also write the exact gradient as
\begin{align}
    \label{eqn: gradient target}
    \nabla_{\xi} F(\xi) = \mathbb{E} \left[ \frac{\phi(s, a, s')}{P_{\xi}(s' | s, a)} \frac{r_{\tau}(s, a) + \gamma V_{P_{\xi}}^{\pi_{\theta^*}}(s')}{1 - \gamma}\right]
\end{align}
where the expectation is taken over \(\nu_{P_\xi}^{\pi_{\theta^*}} \times P_{\xi}\). Next, we show this gradient is Lipschitz smooth.
\begin{lemma}[Lipschitz Smoothness of the Outer Objective]
\label{lemma: outer_smoothness}
Assume that the feature map \(\phi\) is bounded and that the regularity conditions of Lemma~~\ref{lemma: lipschitz and lipschitz smoothness} hold uniformly over \(\theta \in \Theta\). Then the outer objective \(F(\xi)\) is differentiable on \(\Xi\), and its gradient is Lipschitz continuous. In particular, there exists \(L_F > 0\) such that for all \(\xi_1, \xi_2 \in \Xi\),
\begin{align}
\|\nabla_\xi F(\xi_1) - \nabla_\xi F(\xi_2)\|_2
\le L_F \|\xi_1 - \xi_2\|_2.
\end{align}
\end{lemma}

We provide the full proof in Appendix~\ref{appendix: properties of entropy regularized DRMDPs}. Although \(\theta^*\) is unique only up to equivalence classes, the gradient with respect to \(\xi\) is unaffected by variations along equivalence-class directions of \(\theta\), since the objective depends only on the induced policy \(\pi_{\theta^*}\). Consequently, the mapping \(\xi \to \nabla_\xi F(\xi)\) is Lipschitz smooth. 

%% file: Sections/s5_algorithms.tex
\section{Algorithms}
\label{sec:algorithm}

In this section, we present the Robust RAMDP Projected Gradient Descent algorithm and its convergence analysis for s-rectangular MDPs. Then, we extend this to the non-rectangular setting with a Frank-Wolfe algorithm.

\subsection{Robust RAMDP Projected Gradient Descent for S-Rectangular MDPs}
The procedure is detailed in Algorithm~\ref{alg: pgd}. At each iteration $k$, the algorithm relies on two distinct oracles:
\begin{enumerate}
    \item \textbf{PolicyOracle:}\footnote{We note that any reinforcement learning algorithm that satisfies Assumptions~\ref{assumption: fisher non degeneracy} and~\ref{assumption: log policy smoothness} can be used as the policy oracle.} Given the current transition kernel parameters $\xi_k$, this oracle solves the inner maximization problem to return an \(\epsilon_\theta\)-optimal entropy-regularized policy $\pi_{\theta_k}$ 
    \[
        \|\theta_k - \theta_k^*\| \leq \epsilon_\theta
    \]
    \item \textbf{MLMC Gradient Estimator (Algorithm~\ref{alg: mlmc_gradient_discounted}):} Using the current policy $\pi_{\theta_k}$ and kernel parameters $\xi_k$, the output is an estimate of the gradient $\nabla_\xi J^{\pi_{\theta_k}}_{P_{\xi_k}, \tau}$
    \[
        \left\|\hat{\nabla}_\xi J^{\pi_{\theta_k}}_{P_{\xi_k}, \tau} - \nabla_\xi J^{\pi_{\theta_k}}_{P_{\xi_k}, \tau}\right\| \leq \epsilon_{\mathrm{grad}}
    \]
\end{enumerate}
Following the gradient estimation, the transition parameters are updated via a gradient descent step with step size $\beta$ and subsequently projected onto the uncertainty set $\Xi$ to ensure feasibility.

\begin{algorithm}[h]
\caption{Robust RMDP Projected Gradient Descent Algorithm (s-rectangular uncertainty sets)}
\label{alg: pgd}
\begin{algorithmic}[1]
\REQUIRE 
    Uncertainty set $\Xi$, 
    Initial parameter $\xi_0 \in \Xi$, 
    Step size $\beta$, 
    Maximum iterations $K$, 
    Entropy regularization parameter $\tau$.
    Robust transition kernel parameters $\xi_{\tilde{K}}$.

\STATE \textbf{Initialize} $\xi_0 \in \Xi$.
\FOR{$k = 0, \dots, K-1$}
    \STATE $\pi_{\theta_k} \gets \text{PolicyOracle}(P_{\xi_k}, \tau)$
    \STATE $\hat{v}_k \gets \text{MLMCGradientEstimator}(\pi_{\theta_k}, \xi_k, \tau)$ 
    \STATE $\xi_{k+1} \gets \text{Proj}_{\Xi}\bigl(\xi_k - \beta \hat{v}_k \bigr)$ 
\ENDFOR

\STATE \textbf{Return} $\pi_{\theta_{\tilde{K}}}, \xi_{\tilde{K}}$
\end{algorithmic}
\end{algorithm}

We formally analyze the convergence properties of Algorithm~\ref{alg: pgd}, accounting for potential approximation errors inherent in stochastic gradient estimation and policy optimization. 

\begin{theorem}[Convergence of Algorithm~\ref{alg: pgd}]
\label{theorem: pgd_convergence}
Let $\hat{\nabla} J_{P_\xi}^{\pi_{\theta}}$ be an estimator of $\nabla_\xi J_{P_\xi}^{\pi_\theta}$ satisfying \(\|\hat{\nabla} J_{P_\xi}^{\pi_\theta} - \nabla_\xi J_{P_\xi}^{\pi_\theta}\| \leq \epsilon_{\mathrm{grad}}\) and assume the policy error satisfies \(\|\theta_k - \theta_k^*\| \le \epsilon_\theta \quad \forall k\). Then, choosing \(\beta = \frac{1}{2 L_F}, \epsilon_{\mathrm{grad}} = O(\epsilon(1-\gamma)), \epsilon_\theta = O(\epsilon(1-\gamma)^4), K = O(\epsilon^{-3}(1-\gamma)^{-8})\) in Algorithm~\ref{alg: pgd} guarantees
\begin{align}
    F(\xi_{\tilde{K}}) - \min_{\xi \in \Xi} F(\xi) \le \varepsilon.
\end{align}
where \(\xi_{\tilde{K}}\) is the output of Algorithm~\ref{alg: pgd}.
\end{theorem}

\paragraph{Proof Sketch}
We provide the full proof in Appendix~\ref{appendix: algorithm}. The proof of Theorem~\ref{theorem: pgd_convergence} relies on the smoothness of the robust objective, controlled oracle errors, and a gradient dominance property specific to robust MDPs. Using Lemma~\ref{lemma: outer_smoothness}, the outer objective \(F(\xi)\) is shown to be \(L_F\)-smooth, which enables a descent argument for projected updates of the form \(\xi_{k+1} = \mathrm{Proj}_{\Xi}(\xi_k - \beta \hat{v}_k)\). Each iteration decreases the objective proportionally to the squared norm of the gradient mapping, up to a noise floor induced by the policy suboptimality and gradient estimation errors, quantified by \(\epsilon_\theta\) and \(\epsilon_{\mathrm{grad}}\). Lemma~\ref{lemma: gradient dominance linear transition kernel} establishes a gradient dominance condition that bounds the global sub-optimality gap by the gradient norm scaled by the distribution mismatch coefficient \(D\). Combining this dominance property with the descent inequality and summing over iterations yields a convergence rate for the average optimality gap. Choosing \(K = O(\epsilon^{-3}(1-\gamma)^{-8})\) iterations and errors of order \(O(\epsilon)\) ensures that the final robust optimality gap is at most \(\epsilon\).

\subsection{Frank-Wolfe Algorithm for Non-Rectangular Uncertainty Sets}

In robust reinforcement learning, the transition kernel uncertainty is typically assumed to be $(s,a)$-rectangular or $s$-rectangular, which allows the robust Bellman operator to decompose over states and actions. Under such structures, the projection onto the uncertainty set $\Xi$ is trivial, since it reduces to independent projections onto action probability simplices. However, in many realistic scenarios, such as global resource constraints or shared physical dynamics, the uncertainty set is \emph{non-rectangular}. The non-rectangular case differs fundamentally because the transition parameters $\xi$ are coupled across the state space. This has two major implications:
\begin{enumerate}
\item \textbf{Computational bottleneck of projection.}
Projected gradient descent requires evaluating
\[
\mathrm{Proj}_{\Xi}(\xi_k - \beta \nabla F(\xi_k)),
\]
which entails solving a constrained quadratic program over the full parameter space $\Xi$ at every iteration. This becomes prohibitively expensive in high dimensions. Frank--Wolfe replaces this step with a tractable linear program over $\Xi$, which is substantially cheaper.

\item \textbf{Global computational hardness.}
In the $s$-rectangular case, the outer objective $F(\xi)$ admits strong structural properties such as separability, convexity, or tractable Bellman operators, which allow polynomial-time solution methods. In contrast, for general non-rectangular uncertainty sets, the robust planning problem is known to be NP-hard due to the global coupling of transition parameters across states and actions.
\end{enumerate}

\begin{algorithm}[ht]
\caption{Robust RMDP Frank--Wolfe Algorithm (non-rectangular uncertainty sets)}
\label{alg: frank_wolfe}
\begin{algorithmic}[1]
\STATE \textbf{Require:} Uncertainty set $\Xi$, initial point $\xi_0 \in \Xi$, curvature constant $C_f$, number of iterations $K$, $\xi_0 \in \Xi$
\FOR{$k = 0, \ldots, K-1$}
    \STATE $\pi_{\theta_k} \leftarrow \text{PolicyOracle}(P_{\xi_k}, \tau)$
    \STATE $\hat v_k \leftarrow \text{MLMCGradientEstimator}(\pi_{\theta_k}, \xi_k, \tau)$
    \STATE $s_k \leftarrow \arg\min_{s \in \Xi} \langle s, \hat v_k \rangle$
    \STATE $d_k \leftarrow s_k - \xi_k$
    \STATE $\hat g_k \leftarrow \langle \xi_k - s_k, \hat v_k \rangle$
    \STATE $\gamma_k \leftarrow \min\{1, \hat g_k / C_f\}$
    \STATE $\xi_{k+1} \leftarrow \xi_k + \gamma_k d_k$
\ENDFOR
\STATE \textbf{Return:} $\pi_{\theta_K}, \xi_K$
\end{algorithmic}
\end{algorithm}

\begin{theorem}[Convergence of Algorithm~\ref{alg: frank_wolfe}]
\label{theorem: fw_convergence}
Let $\hat{\nabla} F(\xi)$ be an estimator of the gradient $\nabla F(\xi)$ satisfying $\|\hat{\nabla} F(\xi) - \nabla F(\xi)\| \leq l_\pi \epsilon_\theta + \epsilon_{\mathrm{grad}}$. Assume the objective $F$ has a finite bounded curvature constant $C_f$ over the domain $\Xi$ with diameter $D_{\Xi}$. Then, choosing $K = O(C_f^2\epsilon^{-2}(1-\gamma)^{-2})$, $\epsilon_{\mathrm{grad}} = O(\epsilon(1-\gamma))$, and $\epsilon_\theta = O(\epsilon(1-\gamma)^4)$ guarantees that Algorithm~\ref{alg: frank_wolfe} yields:
\begin{align}
    F(\xi_{\tilde{K}}) - \min_{\xi \in \Xi} F(\xi) \leq \epsilon + \frac{D\delta_{\Xi}}{1 - \gamma},
\end{align}
where $\delta_{\Xi}$ is the irreducible error due to the non-rectangularity of $\Xi$ and \(\xi_{\tilde{K}}\) is the output of Algorithm~\ref{alg: frank_wolfe}.
\end{theorem}

\paragraph{Proof sketch.}
We provide the full proof in Appendix~\ref{appendix: algorithm}. The proof of Theorem~\ref{theorem: fw_convergence} shows that local descent guarantees for the Frank--Wolfe gap translate into global optimality even though \(F\) is non-convex and the uncertainty set \(\Xi\) is non-rectangular. Although \(F\) is non-convex, its smoothness ensures that each Frank--Wolfe step yields sufficient descent proportional to the squared gap via the curvature constant \(C_f\) \citep{lacoste2016convergence}, and small policy suboptimality or gradient estimation errors do not destroy this descent. While a vanishing Frank--Wolfe gap typically implies only stationarity, Lemma~\ref{lemma: non-rectangular gradient dominance} establishes a gradient dominance property for robust MDPs with non-rectangular uncertainty, allowing global suboptimality to be controlled by the gap. The remaining effect of non-rectangularity appears as an irreducible approximation error \(\delta_{\Xi}\), so the algorithm converges in \(O(C_f^2 \epsilon^{-2}(1-\gamma)^{-2})\) iterations to a stationary point that is globally \(\epsilon + \frac{D\delta_{\Xi}}{1 - \gamma}\)-optimal.

\begin{lemma}[Gradient Dominance for Non-rectangular Uncertainty sets]
    \label{lemma: non-rectangular gradient dominance}
    For a fixed policy \(\pi_\theta\), let
    \begin{align*}
        \delta_{\Xi} = \min_{\xi' \in \Xi} \left[\max_{\xi_s \in \Xi_s} \xi_s^\top \nabla J_{P_{\xi'}}^{\pi_{\theta}} - \max_{\xi \in \Xi} \xi^\top \nabla J_{P_{\xi'}}^{\pi_{\theta}}\right]
    \end{align*}
    be the degree of non-rectangularity in the linearly parameterized uncertainty set, where \(\Xi_s\) is the smallest s-rectangular uncertainty set within \(\Xi\). Then the following gradient domination theorem holds
    \begin{align*}
        \min_{\xi' \in \Xi} \left(J_{P_{\xi}}^{\pi_{\theta}} - J_{P_{\xi'}}^{\pi_{\theta}}\right) \leq \frac{D}{1 - \gamma} \left[\max_{\xi' \in \Xi} (\xi - \xi')^\top \nabla J_{P_{\xi}}^{\pi_{\theta}} + \delta_{\Xi}\right]
    \end{align*}
    where \(D\) is the distribution mismatch coefficient.
\end{lemma}

%% file: Sections/s6_stochastic_gradient_estimation.tex
\section{Stochastic Gradient Estimation}

In this section, we describe our approach for estimating the gradient of the entropy-regularized objective, $\nabla_\xi F(\xi)$, through our MLMC Algorithm (Algorithm~\ref{alg: mlmc_gradient_discounted}).

\subsection{Single Step Estimate}

The gradient of the dual objective depends on the entropy-regularized value function \(V_{P_\xi, \tau}^{\pi_{\theta^*}}\). For a single transition \(Z_t = (s_t, a_t, s_{t+1})\), the per-transition gradient estimator is
\begin{align*}
    \nabla_\xi F(\xi, Z_t) = \phi(Z_t)
    \left( \frac{r_\tau(s_t, a_t) + \gamma \hat{V}_{P_\xi, \tau}^{\pi_{\theta^*}}(s_{t+1})}{(1 - \gamma)P_\xi(s_{t+1} \mid s_t, a_t)} \right),
\end{align*}
where \(\hat{V}\) denotes an estimate of the value function, which can be obtained via standard actor-critic or temporal-difference methods with linear \citep{li2024high} or neural function approximation \citep{ke2024improved}, achieving a sample complexity of \(\mathcal{O}(\epsilon^{-2}(1 - \gamma)^{-2})\) for an \(\epsilon\)-accurate estimate.

\subsection{Multilevel Monte Carlo for Infinite Horizons}

A fundamental difficulty in estimating $\nabla_\xi F(\xi)$ lies in the infinite‑horizon nature of the value function. Standard Monte Carlo estimators approximate the infinite sum by truncating trajectories at a fixed horizon $H$, leading to $\mathcal{O}(\epsilon^{-4})$ sample complexity \citep{pmlr-v235-chen24s}. To overcome this inefficiency, we adopt a Multilevel Monte Carlo (MLMC) approach. MLMC reduces computational cost by exploiting a telescoping decomposition over a hierarchy of geometrically increasing simulation horizons. An MLMC estimator truncated at $T_{\max}$ is
\begin{align}
    g_{\mathrm{MLMC}} = g^0 + (g^Q - g^{Q-1}) \mathbb{I}\left(2^Q \leq T_\mathrm{max}\right)
\end{align}
where $Q\sim\mathrm{Geom}(1/2)$ and we use the convention $g^{-1}\equiv 0$ and \(g^j\) is a time-averaged approximation of the gradient over a rollout of length \(2^j\). Randomizing over geometric levels in this way preserves the telescoping identity in \citet{blanchet2019unbiased} so the estimator approximates the infinite‑horizon gradient by only using \(\mathcal{O}\left(\log T_{\max}\right)\) samples on average.

\subsection{Robustness and Parallelism}

To ensure robustness against outliers and heavy-tailed noise, we aggregate $N$ independent MLMC samples using a geometric median-of-means (MoM) estimator with $K$ blocks. This combination offers a distinct architectural advantage over standard sequential Monte Carlo. While standard MC requires $O(1/\epsilon)$ sequential environment steps for a single long-horizon trajectory, the expected length of an MLMC rollout is small and near constant, which can be accelerated through parallelism on distributed cores in many existing RL libraries \citep{jaxrl, lu2022discovered}.

\subsection{Sample Complexity}

The MLMC construction ensures that bias is controlled by the truncation level $T_{\max}$, while sample complexity grows only logarithmically with the horizon. We formalize this in the following result, with the full proof in Appendix~\ref{appendix: stochastic gradient estimation}. We note that to provide guarantees we require a minimum transition probability, which can be done via regularization. This is also assumed in \citet{pmlr-v235-chen24s}.

\begin{lemma}[Sample Complexity of MLMC Gradient Estimator]
\label{lemma: mlmc_sample_complexity}
Suppose the value-function approximation satisfies $\mathbb{E}_{s \sim d_{P_\xi}^{\pi_{\theta^*}}} [(\widehat{V}_{P_\xi, \tau}^{\pi_{\theta^*}}(s) - V_{P_\xi, \tau}^{\pi_{\theta^*}}(s))^2] \leq \epsilon_v^2$. Then with probability at least $1 - \beta$, the MoM estimator satisfies $\|\nabla_\xi F_{\mathrm{MoM}}(\xi) - \nabla_\xi F(\xi)\|_2^2 \leq \epsilon^2$ with an expected sample complexity of \(\mathbb{E}[T_{\mathrm{total}}] = \mathcal{O}\left(\epsilon^{-2} (1 - \gamma)^{-4} \log \frac{1}{\beta} \right)\).
\end{lemma}

% \textbf{Proof Sketch of Lemma~\ref{lemma: mlmc_sample_complexity}.} The proof leverages the properties of MLMC estimators established in \citet{ganesh2025a} to achieve a fundamental $\mathcal{O}(\epsilon^{-2})$ factor improvement over standard fixed-horizon methods. We first bound the total bias by decomposing it into the error from the value function oracle $\epsilon_v$ and the geometric truncation error at $T_{\max} = \mathcal{O}(\epsilon^{-2})$. To address the potentially heavy-tailed nature of the resulting $2^Q$-weighted samples, we apply the Geometric Median-of-Means (MoM) estimator, which provides more robust high probability bound required for Algorithms~\ref{alg: pgd} and~\ref{alg: frank_wolfe} \citep{minsker2015geometric}. Combining these together yields a total expected sample complexity of $\tilde{\mathcal{O}}(\epsilon^{-2}(1-\gamma)^{-4})$.

%% file: Sections/s7_theoretical_analysis.tex
\section{Global Convergence Guarantees}
In this section, we provide the full global convergence guarantees for each algorithm. For optimal sample complexity guarantees, we use the NPG algorithm in \citet{mondal2024improved} as the policy oracle for the discounted setting and the SRVR-PG algorithm in \citet{liu2020improved} for the average reward setting.

\begin{theorem}[Discounted Robust MDP Sample Complexity]
\label{theorem: discounted_sample_complexity}
Let \(\epsilon > 0\) and let \(\pi_{\theta_{\tilde{K}}}\) and \(\xi_{\tilde{K}}\) be the outputs from either Algorithm~\ref{alg: pgd} or Algorithm~\ref{alg: frank_wolfe}. Choose
\(
\tau = \Theta \big(\epsilon(1-\gamma)\big), \quad
\epsilon_{\mathrm{model}} = \Theta\big(\epsilon(1-\gamma)\big), \quad
\epsilon_{\mathrm{policy}} = \Theta\big(\epsilon(1-\gamma)^4\big), \quad
\epsilon_{\mathrm{grad}} = \Theta\big(\epsilon(1-\gamma)\big).
\)
Then the discounted robust MDP solution \(\xi_{\tilde K}\) satisfies
\[
J_{P_{\xi_{\tilde{K}}}}^{\pi_{\theta_{\tilde{K}}}} - \min_{\xi \in \Xi}  \max_{\theta \in \Theta} J_{P_\xi}^{\pi_\theta} \le \epsilon.
\]
Moreover, the total sample complexity satisfies
\[
T_{\mathrm{total}} =
\begin{cases}
\mathcal{O}\big(\epsilon^{-5}(1-\gamma)^{-18}\big), & \text{Algorithm}~\ref{alg: pgd}, \\
\mathcal{O}\big(C_f^2 \epsilon^{-4}(1-\gamma)^{-12}\big), & \text{Algorithm}~\ref{alg: frank_wolfe}.
\end{cases}
\]
\end{theorem}

Theorem~\ref{theorem: discounted_sample_complexity} gives the full sample complexity for both the PGD and Frank-Wolfe algorithms, where the Frank-Wolfe algorithm naturally extends to the non-rectangular case with \((\epsilon+\frac{D \delta_\Xi}{1 - \gamma})\)-optimality. We now extend this to the average reward case, with the full proof in Appendix~\ref{appendix: global convergence guarantees}.

\begin{corollary}[Average-Reward Robust MDP Sample Complexity]
Let \(\epsilon > 0\) under the same setting as Theorem~\ref{theorem: discounted_sample_complexity}. Consider the average-reward robust MDP obtained via the reduction technique with discount factor \(\gamma = 1 - \Theta\left(\frac{\epsilon}{H}\right)\) in \citep{wang2022near}. Using the Frank--Wolfe algorithm with the MLMC gradient estimator and SRVR-PG as the policy oracle \citep{liu2020improved}, the returned solution \(\xi_{\tilde K}\) satisfies
\[
g_{P_{\xi_{\tilde{K}}}}^{\pi_{\theta_{\tilde{K}}}} - \min_{\xi \in \Xi}  \max_{\theta \in \Theta} g_{P_\xi}^{\pi_\theta} \le \epsilon.
\]
Moreover, the total sample complexity is
\[
T_{\mathrm{total}} = \mathcal{O}\left(C_f^2 \epsilon^{-10.5} H^{5.5}\right).
\]
\end{corollary}

%% file: Sections/s8_conclusion.tex
\section{Conclusion}
In this paper, we develop a framework for $s$-rectangular and non-rectangular robust MDPs with general policy parameterization. By reversing the optimization order, our method enables gradient-based optimization over expressive policy classes without tabular projections. We established state-space–independent Lipschitz and smoothness bounds and introduced an entropy-regularized discounted reduction that restores minimax structure in average-reward settings, yielding the first sample complexity guarantees in the average reward setting. Our multilevel Monte Carlo gradient estimator further improves sample efficiency by a factor of \(\tilde{\mathcal{O}}(\epsilon^{-2})\) over fixed-horizon methods. 
% Future work could improve our results through more effective average-reward to discounted reduction techniques and extensions beyond linear transition kernel parameterization.

%% file: Sections/Appendix/related_work.tex
\section{Related Work}
\label{appendix: related_work}
\subsection{Robust Markov Decision Processes}
The complexity of solving Robust MDPs is governed primarily by the structure of the uncertainty set \(\mathcal{P}\). Existing work has largely focused on the \((s, a)\)-rectangular setting, where the uncertainty set is independent across state-action pairs and is often represented by divergence measures, such as KL, TV, or Wasserstein distance \cite{nilim2003robustness, iyengar2005robust, li2022first, pmlr-v202-wang23am, kumar2023policy, sun2024policy, xu2025efficient, chen2025sample}. While this independence assumption preserves the validity of the Bellman recurrence and allows for polynomial-time planning, it often yields overly conservative policies that do not reflect realistic environments where dynamics are coupled. To address this, $s$-rectangular sets were introduced, allowing uncertainty to be coupled across actions within a specific state while maintaining independence across states \citep{wiesemann2013robust, li2023first, pmlr-v235-chen24s, li2023policy, kumar2022efficient, kumar2023policy, kumar2023towards, pmlr-v202-wang23i, wang2025provable}. This has also been further extended to non-rectangular uncertainty sets coupled on both states and actions, but have been proven to be NP-Hard \citep{grand2023beyond, li2023policy, wang2025provable}. However, all of these methods are limited to the tabular or linear MDP setting. Recently, \citet{ghosh2025scaling} has developed a value iteration algorithm with general function approximation, but only in the weaker episodic, discounted, \((s, a)\)-rectangularity setting with tabular policies. To the best of our knowledge, we propose the first algorithm that extends to $s$-rectangular and non-rectangular MDPs with general policy parameterization and infinite state spaces in both the discounted and average reward settings.

\subsection{Markov Decision Processes with General Policy Parameterization}
Reinforcement Learning in large or infinite state-action spaces necessitates moving beyond tabular representations toward general parameterization, where policies are indexed by a low-dimensional parameter vector. While classical model-based algorithms such as UCRL2 \citep{jaksch2010near} and PSRL \citep{shipra2017optimistic, agrawal2024optimistic} offer strong regret guarantees, their reliance on explicit state-action representations restricts their applicability to tabular settings. Within this framework, two primary methodological tracks have emerged. The first consists of policy gradient methods, which directly estimate gradients of the long-run average reward using trajectory-level sampling \citep{mondal2024improved, bai2024regret, ganesh2025order}. These methods can attain the optimal convergence rate of \(\tilde{\mathcal{O}}(T^{-1/2})\), but their variance scales poorly with the action space and their analysis often depends on explicit knowledge of mixing or hitting times, assumptions that are difficult to verify or satisfy under general parameterization. The second track comprises actor-critic methods, which introduce a temporal-difference-based critic to reduce gradient variance. Recent work has established global convergence guarantees for actor-critic algorithms in the average-reward setting under general parameterization \citep{suttle2023beyond, patel2024towards, wang2024non, gaur2024closing, ganesh2025a, ganesh2025orderneural}. However, all of these algorithms are derived in the non-robust setting and fundamentally rely on accurate model specification, offering no theoretical protection against model misspecification or adversarial uncertainty.

%% file: Sections/Appendix/properties.tex
\section{Properties of Entropy Regularized Discounted Reward MDPs}
\label{appendix: properties of entropy regularized DRMDPs}
\begin{lemma}[Bounds for Discounted Entropy-Regularized State-Value Function]
Let $V_{P, \tau}^{\pi_\theta}(s)$ denote the discounted entropy-regularized state-value function for a policy $\pi_\theta$ under transition kernel $P$, with per-step reward $r_\tau(s,a)$ and discounted return $J_{P, \tau}^{\pi_\theta}$. Then for all $s \in \mathcal{S}$ and $a \in \mathcal{A}$, we have
\begin{align}
    V_{P, \tau}^{\pi_\theta}(s) &\leq \frac{1+\tau \log |\mathcal{A}|}{1-\gamma}.
\end{align}
\end{lemma}

\begin{proof}
From the discounted state-value function definition:
\begin{align}
    V_{P, \tau}^{\pi_\theta}(s) &= \mathbb{E}_{\pi_\theta,P} \left[ \sum_{t=0}^{\infty} \gamma^t r_{t, \tau} \Big|\, s_0 = s \right] \\
    &= \frac{1}{1 - \gamma} \sum_{s'} \sum_{a'}d_{\pi_\theta}^P(s') \pi_\theta(a' \mid s') r_\tau(s', a') \\
    &\leq \frac{1 + \tau \log |\mathcal{A}|}{1 - \gamma}
\end{align}
where we use the upper bound of the policy entropy in the inequality.
\end{proof}

\begin{lemma}[Transition Kernel Performance Difference Lemma]
\label{lemma: transition kernel performance difference}
Let $\pi_\theta$ be a fixed policy, and let $P$ and $P'$ be two transition kernels. Then the difference in discounted returns satisfies
\[
J_{P,\tau}^{\pi_\theta} - J_{P',\tau}^{\pi_\theta} = \sum_{s \in \mathcal{S}} \sum_{a \in \mathcal{A}} d_{\pi_\theta}^P(s)\, \pi_\theta(a \mid s) \sum_{s' \in \mathcal{S}} \left(P(s' \mid s,a) - P'(s' \mid s,a)\right) V_{P',\tau}^{\pi_\theta}(s'),
\]
where $d_{\pi_\theta}^P(s) = (1-\gamma) \sum_{t=0}^{\infty} \gamma^t \Pr(s_t = s \mid \pi_\theta, P)$ is the discounted state visitation distribution under $P$.
\end{lemma}

\begin{proof}
Starting from the discounted Bellman equation for $P'$:
\[
    V_{P',\tau}^{\pi_\theta}(s) = \sum_a \pi_\theta(a \mid s) \Big[ r_\tau(s,a) + \gamma \sum_{s'} P'(s' \mid s,a) V_{P',\tau}^{\pi_\theta}(s') \Big].
\]
Take the expectation over the discounted visitation distribution $d_{\pi_\theta}^P(s) \pi_\theta(a \mid s)$:
\begin{align*}
    &\sum_s \sum_a d_{\pi_\theta}^P(s) \pi_\theta(a \mid s) r_\tau(s,a) \\
    &= \sum_s \sum_a d_{\pi_\theta}^P(s) \pi_\theta(a \mid s) Q_{P',\tau}^{\pi_\theta}(s,a) - \sum_s \sum_a d_{\pi_\theta}^P(s) \pi_\theta(a \mid s) \sum_{s'} P'(s' \mid s,a) \gamma V_{P',\tau}^{\pi_\theta}(s').
\end{align*}
By definition of the discounted return under $P$:
\[
    \sum_s \sum_a d_{\pi_\theta}^P(s) \pi_\theta(a \mid s) Q_{P',\tau}^{\pi_\theta}(s,a) = \sum_s \sum_a d_{\pi_\theta}^P(s) \pi_\theta(a \mid s) \sum_{s'} P(s' \mid s,a) \gamma V_{P',\tau}^{\pi_\theta}(s') + J_{P',\tau}^{\pi_\theta}.
\]
Subtracting terms gives
\[
    J_{P,\tau}^{\pi_\theta} - J_{P',\tau}^{\pi_\theta} = \sum_s \sum_a d_{\pi_\theta}^P(s) \pi_\theta(a \mid s) \sum_{s'} \left(P(s' \mid s,a) - P'(s' \mid s,a)\right) V_{P',\tau}^{\pi_\theta}(s'),
\]
\end{proof}

\begin{lemma}[Discounted State-Visitation Lipschitz Bounds]
\label{lemma: discounted visitation lipschitz}
Let $P$ be a Markov transition kernel over a finite state space $\mathcal{S}$, and let $\pi_\theta, \pi_{\theta'}$ be two policies. Let 
\[
d_{\pi_\theta}^P(s) := (1-\gamma) \sum_{t=0}^{\infty} \gamma^t \Pr(s_t = s \mid \pi_\theta, P)
\] 
denote the discounted state visitation distribution under policy $\pi_\theta$ and kernel $P$. Then
\[
\| d_{\pi_{\theta'}}^P - d_{\pi_\theta}^P \|_1 \leq \frac{1}{1-\gamma} \max_{s \in \mathcal{S}} \| \pi_{\theta'}(\cdot \mid s) - \pi_\theta(\cdot \mid s) \|_1.
\]

Similarly, for two transition kernels $P, P'$ under the same policy $\pi_\theta$,
\[
    \| d_{\pi_\theta}^{P'} - d_{\pi_\theta}^{P} \|_1 \leq \frac{1}{1-\gamma} \max_{s \in \mathcal{S}, a \in \mathcal{A}} \| P'(\cdot \mid s,a) - P(\cdot \mid s,a) \|_1.
\]
\end{lemma}

\begin{proof}
For two policies $\pi_\theta, \pi_{\theta'}$, define the signed measure
\[
    \mu(s') := \sum_s \sum_a d_{\pi_\theta}^P(s) \big(\pi_{\theta'}(a \mid s) - \pi_\theta(a \mid s)\big) P(s' \mid s, a), \qquad \sum_{s'} \mu(s') = 0.
\]
Then the discounted visitation difference satisfies
\[
    d_{\pi_{\theta'}}^P - d_{\pi_\theta}^P = \mu + \gamma (d_{\pi_{\theta'}}^P - d_{\pi_\theta}^P) P_{\pi_{\theta'}},
\]
where $P_{\pi_{\theta'}}$ is the one-step kernel under $\pi_{\theta'}$. Iterating,
\[
    d_{\pi_{\theta'}}^P - d_{\pi_\theta}^P = \sum_{k=0}^{\infty} \gamma^k \mu P_{\pi_{\theta'}}^{(k)}.
\]
Taking total variation norms and using $\|\mu P_{\pi_{\theta'}}^{(k)}\|_1 \le \|\mu\|_1$ gives
\[
    \| d_{\pi_{\theta'}}^P - d_{\pi_\theta}^P \|_1 \leq \sum_{k=0}^{\infty} \gamma^k \|\mu\|_1 = \frac{1}{1-\gamma} \|\mu\|_1 \leq \frac{1}{1-\gamma} \max_s \|\pi_{\theta'}(\cdot \mid s) - \pi_\theta(\cdot \mid s)\|_1,
\]
Similarly, for two transition kernels $P, P'$ under the same policy $\pi_\theta$, define
\[
    \mu'(s') := \sum_s \sum_a d_{\pi_\theta}^P(s) \pi_\theta(a \mid s) \left(P'(s' \mid s,a) - P(s' \mid s,a)\right), \qquad \sum_{s'} \mu'(s') = 0.
\]
Then
\[
    d_{\pi_\theta}^{P'} - d_{\pi_\theta}^{P} = \sum_{k=0}^{\infty} \gamma^k \mu' P_{\pi_\theta}^{(k)}, \quad
    \| d_{\pi_\theta}^{P'} - d_{\pi_\theta}^{P} \|_1 \le \frac{1}{1-\gamma} \|\mu'\|_1 \leq \frac{1}{1-\gamma} \max_{s,a} \| P'(\cdot \mid s,a) - P(\cdot \mid s,a) \|_1.
\]
\end{proof}

\begin{lemma}[Discounted Wasserstein Model Error Bound]
\label{lemma: discounted model error bound}
Assume that Assumption~\ref{assumption: transition kernel model error} holds and the value function $V_{P',\tau}^{\pi_\theta}$ is $L_V$-Lipschitz with respect to the state distance metric.
Then
\[
    \left| J_{P,\tau}^{\pi_\theta}-J_{P',\tau}^{\pi_\theta} \right| \leq L_V\,\varepsilon_{\mathrm{model}} .
\]
\end{lemma}

\begin{proof}
From Lemma~\ref{lemma: transition kernel performance difference} for the discounted setting,
\[
    J_{P,\tau}^{\pi_\theta} - J_{P',\tau}^{\pi_\theta} = \frac{1}{1-\gamma} \sum_s \sum_a d_{\pi_\theta}^P(s)\pi_\theta(a\mid s) \int V_{P',\tau}^{\pi_\theta}(s') \big(P-P'\big)(ds'\mid s,a),
\]
where $d_{\pi_\theta}^P$ denotes the discounted state visitation distribution under $P$. Taking absolute values and applying Jensen's inequality,
\[
    \left| J_{P,\tau}^{\pi_\theta}-J_{P',\tau}^{\pi_\theta} \right| \leq \frac{1}{1-\gamma} \sum_s \sum_a d_{\pi_\theta}^P(s)\pi_\theta(a\mid s) \left| \int V_{P',\tau}^{\pi_\theta}(s') \left(P-P'\right)(ds'\mid s,a) \right|.
\]
Since $V_{P',\tau}^{\pi_\theta}$ is $L_V$-Lipschitz, Kantorovich--Rubinstein duality implies
\[
    \left| \int V_{P',\tau}^{\pi_\theta}(s') \left(P - P'\right)(ds'\mid s,a) \right| \leq L_V W_1\left(P(\cdot\mid s,a),P'(\cdot\mid s,a)\right).
\]
Using the uniform Wasserstein model error bound,
\[
    \left| \int V_{P',\tau}^{\pi_\theta}(s') \left(P - P'\right)(ds'\mid s,a) \right| \leq L_V \epsilon_{\mathrm{model}}.
\]
Plugging this into the previous inequality yields
\[
    \left| J_{P,\tau}^{\pi_\theta}-J_{P',\tau}^{\pi_\theta} \right| \leq \frac{L_V\epsilon_{\mathrm{model}}}{1-\gamma} .
\]
\end{proof}

\begin{lemma}
    Under Assumption~\ref{assumption: transition kernel model error} and if the value function $V_{P_\xi, \tau}^{\pi_\theta}$ is $L_V$-Lipschitz uniformly with respect to the state distance metric, for any $\epsilon \geq 0$ and $\tau > 0$ an $(\epsilon, \tau)$-Nash Equilibrium exists. If $(\theta, \xi)$ is such an $(\epsilon, \tau)$-Nash Equilibrium, then $\pi_\theta$ is an $(2\epsilon + \frac{\tau \log |\mathcal{A}|}{1 - \gamma} + \frac{2L_V \epsilon_{\mathrm{model}}}{1-\gamma})$-optimal policy to the original discounted robust optimization problem.
\end{lemma}

\begin{proof}
We consider the zero-sum game defined entirely over the parameter spaces $\Xi$ and $\Theta$.
Fix $\tau > 0$. Since $\Xi$ and $\Theta$ are compact and convex, the entropy-regularized objective $J_{P_\xi,\tau}^{\pi_\theta}$ is continuous in both arguments. Define the best-response value function
\[
    F(\xi) \coloneqq \max_{\theta \in \Theta} J_{P_\xi,\tau}^{\pi_\theta}.
\]
Because $P_\xi$ depends linearly on $\xi$ and the policy optimization problem is convex due to entropy regularization, the minimizer is unique and $F(\xi)$ is concave in $\xi$.
Since $\Xi$ is compact, there exists
\[
    \xi^* \in \argmin_{\xi \in \Xi} F(\xi).
\]
Let
\[
    \theta^* \coloneqq \argmax_{\theta \in \Theta} J_{P_{\xi^*},\tau}^{\pi_\theta}.
\]
We verify that $(\xi^*,\theta^*)$ satisfies the Nash conditions.
First, by construction,
\[
    J_{P_{\xi^*},\tau}^{\pi_{\theta^*}} - \max_{\theta' \in \Theta} J_{P_{\xi^*},\tau}^{\pi_{\theta'}} = 0 \leq \epsilon.
\]
Second, since $\xi^*$ minimizes $F$,
\[
    \min_{\xi' \in \Xi} J_{P_{\xi'},\tau}^{\pi_{\theta^*}}
    - J_{P_{\xi^*},\tau}^{\pi_{\theta^*}} = 0 \leq \epsilon.
\]
Thus, there exists an $(\epsilon,\tau)$-Nash equilibrium for all $\epsilon \geq 0$ and $\tau > 0$. We now show the optimality of the resulting policy by first finding a lower bound on the robust value of the equilibrium policy:
\begin{align}
    J_{\mathcal{P}}^{\pi_\theta} = \min_{P \in \mathcal{P}} J_P^{\pi_\theta} \geq \min_{\xi' \in \Xi} J_{P_{\xi'}}^{\pi_\theta} - \frac{L_V\epsilon_{\mathrm{model}}}{1-\gamma} \geq \min_{\xi' \in \Xi} J_{P_{\xi'},\tau}^{\pi_\theta} - \frac{L_V \epsilon_{\mathrm{model}}}{1 - \gamma}.
\end{align}
Using the transition kernel Nash condition,
\begin{align}
    \min_{\xi' \in \Xi} J_{P_{\xi'},\tau}^{\pi_{\theta^*}} \geq    J_{P_{\xi^*},\tau}^{\pi_{\theta^*}} - \epsilon,
\end{align}
we obtain
\begin{equation}
\label{eq:lower_bound_disc}
    J_{\mathcal{P}}^{\pi_{\theta^*}} \geq J_{P_{\xi^*},\tau}^{\pi_{\theta^*}} - \epsilon - \frac{L_V \epsilon_{\mathrm{model}}}{1 - \gamma}.
\end{equation}
Next, we upper bound the optimal robust value:
\begin{align}
    \max_{\theta' \in \Theta} J_{\mathcal{P}}^{\pi_{\theta'}} = \max_{\theta' \in \Theta} \min_{P \in \mathcal{P}} J_P^{\pi_{\theta'}} \leq \max_{\theta' \in \Theta} J_{P_{\xi^*}}^{\pi_{\theta'}}
    + L_V \epsilon_{\mathrm{model}} \leq \max_{\theta' \in \Theta} \left(J_{P_{\xi^*}, \tau}^{\pi_{\theta'}} + \frac{\tau \log |\mathcal{A}|}{1 - \gamma} + \frac{L_V \epsilon_{\mathrm{model}}}{1 - \gamma}\right).
\end{align}
Using the policy Nash condition,
\[
    \max_{\theta' \in \Theta} J_{P_{\xi^*},\tau}^{\pi_{\theta'}} \le J_{P_{\xi^*},\tau}^{\pi_{\theta^*}} + \epsilon,
\]
we obtain
\begin{equation}
\label{eq:upper_bound_disc}
    \max_{\theta' \in \Theta} J_{\mathcal{P}}^{\pi_{\theta'}}
    \le J_{P_{\xi^*},\tau}^{\pi_{\theta^*}} + \epsilon + \frac{\tau \log |\mathcal{A}|}{1 - \gamma} + \frac{L_V \epsilon_{\mathrm{model}}}{1 - \gamma}.
\end{equation}

Subtracting \eqref{eq:lower_bound_disc} from \eqref{eq:upper_bound_disc} yields
\begin{align}
    \max_{\theta' \in \Theta} J_{\mathcal{P}}^{\pi_{\theta'}} - J_{\mathcal{P}}^{\pi_{\theta^*}} \leq 2\epsilon + \frac{\tau \log |\mathcal{A}|}{1 - \gamma} + \frac{2L_V \epsilon_{\mathrm{model}}}{1 - \gamma}.
\end{align}
This shows that $\pi_{\theta^*}$ is $\left(2\epsilon + \frac{\tau \log |\mathcal{A}|}{1 - \gamma} + \frac{2 L_V \epsilon_{\mathrm{model}}}{1 - \gamma}\right)$-optimal for the true robust discounted problem.
\end{proof}

\begin{lemma}[Entropy-Regularized Discounted Value Function Lipschitz Bounds]
\label{lemma: discounted value lipschitz}
Let $\gamma \in [0,1)$. Then the entropy-regularized discounted value function $V_{P_\xi, \tau, \gamma}^{\pi_\theta}$ satisfies the following Lipschitz bounds:
    \[
        \left\| V_{P_\xi, \tau}^{\pi_{\theta_1}} - V_{P_\xi, \tau}^{\pi_{\theta_2}} \right\|_\infty 
        \leq \frac{1}{(1-\gamma)^2} \left(1 + \tau \log |\mathcal{A}|\right) G_1 \sqrt{|\mathcal{A}|} \|\theta_1 - \theta_2\|_2.
    \]
    \[
        \left\| V_{P_{\xi_1}, \tau}^{\pi_\theta} - V_{P_{\xi_2}, \tau}^{\pi_\theta} \right\|_\infty 
        \leq \frac{\gamma}{(1-\gamma)^2} \left(1 + \tau \log |\mathcal{A}|\right) \|P_{\xi_1} - P_{\xi_2}\|_1.
    \]
\end{lemma}

\begin{proof}
We define the entropy regularized immediate reward function \(\tilde{r}^{\pi_\theta}\) and transition operator induced by \(\pi_\theta\) \(\mathbb{P}^{\pi_\theta}\) as
\begin{align}
    \tilde{r}^{\pi_\theta} = \sum_{a \in \mathcal{A}} \pi_\theta(a | s) \left(r(s, a) - \tau \log \pi_\theta(a | s)\right), \quad 
\end{align}
The discounted entropy-regularized value function can be written as
\begin{align}
    V_{P_\xi, \tau}^{\pi_\theta} = (I - \gamma \mathbb{P}^{\pi_\theta})^{-1} \tilde{r}^{\pi_\theta}, \quad \text{with} \quad \|\tilde{r}^{\pi_\theta}\|_\infty \le 1 + \tau \log |\mathcal{A}|.
\end{align}
Also, the resolvent satisfies \(\|(I - \gamma \mathbb{P}^{\pi_\theta})^{-1}\|_\infty \le \frac{1}{1-\gamma}\). Using the resolvent identity \((A^{-1} - B^{-1}) = A^{-1} (B - A) B^{-1}\),
\begin{align}
    V_{P_\xi, \tau}^{\pi_{\theta_1}} - V_{P_\xi, \tau}^{\pi_{\theta_2}}  &= (I - \gamma \mathbb{P}^{\pi_{\theta_1}})^{-1} (\tilde{r}^{\pi_{\theta_1}} - \tilde{r}^{\pi_{\theta_2}}) + \gamma (I - \gamma \mathbb{P}^{\pi_{\theta_1}})^{-1} (\mathbb{P}^{\pi_{\theta_1}} - \mathbb{P}^{\pi_{\theta_2}}) (I - \gamma \mathbb{P}^{\pi_{\theta_2}})^{-1} \tilde{r}^{\pi_{\theta_2}}.
\end{align}
Taking the $\ell_\infty$ norm:
\begin{align}
    \left\|V_{P_\xi, \tau}^{\pi_{\theta_1}} - V_{P_\xi, \tau}^{\pi_{\theta_2}}\right\|_\infty & \leq \frac{1}{1-\gamma} \|\tilde{r}^{\pi_{\theta_1}} - \tilde{r}^{\pi_{\theta_2}}\|_\infty + \frac{\gamma}{1-\gamma} \|\mathbb{P}^{\pi_{\theta_1}} - \mathbb{P}^{\pi_{\theta_2}}\|_\infty \left\|V_{P, \tau}^{\pi_{\theta_2}}\right\|_\infty \\
    &\leq \frac{1}{1-\gamma} (1 + \tau \log |\mathcal{A}|) G_1 \sqrt{|\mathcal{A}|} \|\theta_1 - \theta_2\|_2 \\
    &\quad + \frac{\gamma}{1-\gamma} (G_1 \sqrt{|\mathcal{A}|} \|\theta_1 - \theta_2\|_2) \frac{1 + \tau \log |\mathcal{A}|}{1-\gamma} \\
    &= \frac{1}{(1-\gamma)^2} (1 + \tau \log |\mathcal{A}|) G_1 \sqrt{|\mathcal{A}|} \|\theta_1 - \theta_2\|_2.
\end{align}
  
For a fixed policy $\pi_\theta$, let \(\mathbb{P}_1 = \mathbb{P}_{\xi_1}^{\pi_\theta}\) and \(\mathbb{P}_2 = \mathbb{P}_{\xi_2}^{\pi_\theta}\):
\begin{align}
\left\|V_{P_{\xi_1}, \tau}^{\pi_\theta} - V_{P_{\xi_2}, \tau}^{\pi_\theta}\right\|_\infty 
&= \|\gamma (I - \gamma \mathbb{P}_1)^{-1} (\mathbb{P}_1 - \mathbb{P}_2) (I - \gamma \mathbb{P}_2)^{-1} \tilde{r}\|_\infty \\
&\le \gamma \|(I - \gamma \mathbb{P}_1)^{-1}\|_\infty \|\mathbb{P}_1 - \mathbb{P}_2\|_1 \|(I - \gamma \mathbb{P}_2)^{-1} \tilde{r}\|_\infty \\
&\le \frac{\gamma}{(1-\gamma)^2} (1 + \tau \log |\mathcal{A}|) \|P_{\xi_1} - P_{\xi_2}\|_1.
\end{align}
\end{proof}

\begin{lemma}[Lipschitz Continuity of Optimal Discounted Policy]
\label{lemma:discounted_log_policy_lipschitz}
Let $\pi_\xi^*$ be the optimal policy of an entropy-regularized discounted MDP with transition kernel $P_\xi$, discount factor $\gamma \in [0,1)$, and temperature $\tau > 0$.  
Define the entropy-adjusted Q-function
\[
    Q^{e, \pi_\theta}_{P_\xi, \tau}(s, a) := Q^{\pi_\theta}_{P_\xi, \tau}(s, a) + \tau \log \pi_\theta(a|s),
\]
where \(Q^\pi_{P_\xi, \tau, \gamma}\) is the discounted entropy-regularized Q-function.   Under standard Lipschitz and linear parameterization assumptions, the optimal policy is Lipschitz continuous with respect to the transition kernel. Specifically, for any $\xi_1, \xi_2 \in \Xi$:
\[
    \max_s \|\pi_{\theta_{\xi_1}^*}(\cdot|s) - \pi_{\theta_{\xi_2}^*}(\cdot|s)\|_1 
    \le \frac{2 |\mathcal{A}|}{\tau (1-\gamma)^2} \left(1 + \tau \log |\mathcal{A}|\right) \|P_{\xi_1} - P_{\xi_2}\|_1.
\]
\end{lemma}

\begin{proof}
The optimal discounted entropy-regularized policy satisfies the softmax form
\[
    \pi_{\theta_{\xi}^*}(a|s) = \frac{\exp(Q^{e, \pi_{\theta_{\xi}^*}}_{P_\xi, \tau}(s, a)/\tau)}{\sum_{a'} \exp(Q^{e, \pi_{\theta_{\xi}^*}}_{P_\xi, \tau}({s, a')}/\tau)}
\]
Since the softmax is $\frac{1}{\tau}$-Lipschitz in $\ell_1$, we have for any state $s$:
\begin{align}
\|\pi_{\theta_{\xi_1}^*}(\cdot|s) - \pi_{\theta_{\xi_2}^*}(\cdot|s)\|_1
&\le \frac{1}{\tau} \left\|Q^{e, \pi_{\theta_{\xi_1}^*}}_{P_{\xi_1}, \tau}(s, \cdot) - Q^{e, \pi_{\theta_{\xi_2}^*}}_{P_{\xi_2}, \tau}(s, \cdot)\right\|_1 \\
&\le \frac{2 |\mathcal{A}|}{\tau} \|V_{P_{\xi_1}, \tau}^{\pi_{\theta_{\xi_1}^*}} - V_{P_{\xi_2}, \tau}^{\pi_{\theta_{\xi_2}^*}}\|_\infty \\
&\le \frac{2 |\mathcal{A}|}{\tau (1-\gamma)^2} \left(1 + \tau \log |\mathcal{A}|\right) \|P_{\xi_1} - P_{\xi_2}\|_1.
\end{align}
where the first inequality uses the softmax Lipschitz constant \(\frac{1}{\tau}\), the second inequality uses the standard relation between Q-function differences and the value function differences, and the final inequality applies the discounted value function Lipschitz bound (Lemma~\ref{lemma: discounted value lipschitz}), where the resolvent norm gives the $(1-\gamma)^{-2}$ factor.
\end{proof}

\begin{lemma}
    If Assumption~\ref{assumption: log policy smoothness} holds, the entropy-regularized expected discounted cumulative reward \(J_{P_\xi, \tau}^{\pi_\theta}\) and its gradient with respect to the transition parameters \(\nabla_\xi J_{P_\xi, \tau}^{\pi_\theta}\) satisfy the following Lipschitz properties for any \(\theta_1, \theta_2 \in \Pi\) and \(\xi_1, \xi_2 \in \Xi\):  
    \begin{align}
        \left| J_{P_\xi, \tau}^{\pi_{\theta_1}} - J_{P_\xi, \tau}^{\pi_{\theta_2}} \right| &\leq L_\pi  \| \theta_1 - \theta_2 \|_2,  \\
        \left| J_{P_{\xi_1}, \tau}^{\pi_{\theta}} - J_{P_{\xi_2}, \tau}^{\pi_{\theta}} \right| &\leq L_\xi  \| \xi_1 - \xi_2 \|_2,\\
        \left\| \nabla_\xi J_{P_\xi, \tau}^{\pi_{\theta_1}} - \nabla_\xi J_{P_\xi, \tau}^{\pi_{\theta_2}} \right\|_2 &\leq l_{\pi} \| \theta_1 - \theta_2 \|_2, \\
        \left\| \nabla_\xi J_{P_{\xi_1}, \tau}^{\pi_{\theta}} - \nabla_\xi J_{P_{\xi_1}, \tau}^{\pi_{\theta}} \right\|_2 &\leq l_{\xi} \| \xi_1 - \xi_2 \|_2 , 
    \end{align}
    where \(L_\pi = \frac{2(1 + \tau \log |\mathcal{A}|)}{(1 - \gamma)^2} G_1 \sqrt{|\mathcal{A}|}, L_\xi = \frac{\sqrt{d_p}C_\phi(1+\log |\mathcal{A}|)}{(1-\gamma)^2}, l_\pi = \frac{5 \gamma (1+ \tau \log |\mathcal{A}|)}{(1-\gamma)^3} C_\phi  G_1 \sqrt{|\mathcal{A}|}, l_\xi = \frac{3\gamma}{(1 - \gamma)^3} \sqrt{d_p}(1+\log |\mathcal{A}|) C_\phi\).
\end{lemma}

\begin{proof}
    We first prove Equation~\eqref{eqn: lipschitz theta}. We compute the difference for two policies $\theta_1, \theta_2$:
    \begin{align}
        \left| J_{P_\xi, \tau}^{\pi_{\theta_1}} - J_{P_\xi, \tau}^{\pi_{\theta_2}} \right|
        &= \frac{1}{1-\gamma} \left| \sum_s d^{P_\xi}_{\pi_{\theta_1}}(s) \tilde{r}^{\pi_{\theta_1}}(s) - \sum_s d^{P_\xi}_{\pi_{\theta_2}}(s) \tilde{r}^{\pi_{\theta_2}}(s) \right| \\
        &\leq \frac{1}{1-\gamma} \sum_s \left| d^{P_\xi}_{\pi_{\theta_1}}(s) - d^{P_\xi}_{\pi_{\theta_2}}(s) \right| \, \left| \tilde{r}^{\pi_{\theta_1}}(s) \right|
        + \frac{1}{1-\gamma} \sum_s d^{P_\xi}_{\pi_{\theta_2}}(s) \left| \tilde{r}^{\pi_{\theta_1}}(s) - \tilde{r}^{\pi_{\theta_2}}(s) \right| \\
        &\leq \frac{1 + \tau \log |\mathcal{A}|}{(1 - \gamma)^2} G_1 \sqrt{|\mathcal{A}|} \|\theta_1 - \theta_2\|_2 + \frac{1}{1-\gamma} \|\tilde{r}^{\pi_{\theta_1}}(s) - \tilde{r}^{\pi_{\theta_2}}(s)\|_\infty
    \end{align}
    where the last inequality follows from Lemma~\ref{lemma: discounted visitation lipschitz} and Lemma~\ref{lemma: policy and entropy lipschitz}.
    For each state \(s\):
    \begin{align}
        \left| \tilde{r}^{\pi_{\theta_1}}(s) - \tilde{r}^{\pi_{\theta_2}}(s) \right|
        &= \Biggl| \sum_a \pi_{\theta_1}(a|s) \big[r(s,a)-\tau \log \pi_{\theta_1}(a|s)\big] - \sum_a \pi_{\theta_2}(a|s) \big[r(s,a)-\tau \log \pi_{\theta_2}(a|s)\big] \Biggr| \\
        &\leq (1 + \tau \log |\mathcal{A}|)\sum_a \left| \pi_{\theta_1}(a|s) - \pi_{\theta_2}(a|s) \right| \\
        &\leq (1 + \tau \log |\mathcal{A}|) G_1 \sqrt{|\mathcal{A}|} \|\theta_1 - \theta_2\|_2,
    \end{align}
    where the last inequality uses Lemma~\ref{lemma: policy and entropy lipschitz}.
    \begin{align}
        \left| J_{P_\xi, \tau}^{\pi_{\theta_1}} - J_{P_\xi, \tau}^{\pi_{\theta_2}} \right|
        &\le  \frac{1 + \tau \log |\mathcal{A}|}{(1 - \gamma)^2} G_1 \sqrt{|\mathcal{A}|} \|\theta_1 - \theta_2\|_2 + \frac{1}{1-\gamma} \|\tilde{r}^{\pi_{\theta_1}}(s) - \tilde{r}^{\pi_{\theta_2}}(s)\|_\infty \\
        &\le  \frac{1 + \tau \log |\mathcal{A}|}{(1 - \gamma)^2} G_1 \sqrt{|\mathcal{A}|} \|\theta_1 - \theta_2\|_2 + \frac{1}{1-\gamma} (1 + \tau \log |\mathcal{A}|) G_1 \sqrt{|\mathcal{A}|} \|\theta_1 - \theta_2\|_2 \\
        &\le  \frac{2(1 + \tau \log |\mathcal{A}|)}{(1 - \gamma)^2} G_1 \sqrt{|\mathcal{A}|} \|\theta_1 - \theta_2\|_2 \\
    \end{align}
    For Equation~\eqref{eqn: lipschitz xi}, we use Lemma~\ref{lemma: transition kernel performance difference}
    \begin{align}
        \left| J_{P_{\xi_1}, \tau}^{\pi_{\theta}} - J_{P_{\xi_2}, \tau}^{\pi_{\theta}} \right| &\leq \frac{1}{1-\gamma} \left| \sum_{s} \sum_{a} d^{P_{\xi_1}}_{\pi_\theta}(s) \pi_\theta(a | s) \sum_{s'} \left(P_{\xi_1}(s' | s, a) - P_{\xi_2}(s' | s, a)\right)V_{P_{\xi_2}, \tau}^{\pi_\theta}(s') \right| \\
        &\leq \frac{1}{1-\gamma} \sum_{s} \sum_{a} d^{P_{\xi_1}}_{\pi_\theta}(s) \pi_\theta(a | s) \sum_{s'} \left|P_{\xi_1}(s' | s, a) - P_{\xi_2}(s' | s, a)\right| \left|V_{P_{\xi_2}, \tau}^{\pi_\theta}(s') \right| \\
        &\leq \frac{(1+\log |\mathcal{A}|)}{(1-\gamma)^2}\sum_{s} \sum_{a} d^{P_{\xi_1}}_{\pi_\theta}(s) \pi_\theta(a | s) \left\|P_{\xi_1}(\cdot | s, a) - P_{\xi_2}(\cdot | s, a)\right\|_1 \\
        &\leq \frac{(1+\log |\mathcal{A}|)}{(1-\gamma)^2}\sum_{s} \sum_{a} d^{P_{\xi_1}}_{\pi_\theta}(s) \pi_\theta(a | s) \left\| \sum_{i=1}^{d_p} \phi_i(s, a, \cdot) \left(\xi^i_1 - \xi^i_2 \right)\right\|_1 \\
        &\leq \frac{(1+\log |\mathcal{A}|)}{(1-\gamma)^2}\sum_{s} \sum_{a} d^{P_{\xi_1}}_{\pi_\theta}(s) \pi_\theta(a | s)  \sum_{i=1}^{d_p} \left\| \phi_i(s, a, \cdot)\right\|_1 \left(\xi^i_1 - \xi^i_2 \right) \\
        &\stackrel{(a)}{\leq} \frac{C_\phi(1+\log |\mathcal{A}|)}{(1-\gamma)^2} \sum_{s} \sum_{a} d^{P_{\xi_1}}_{\pi_\theta}(s) \pi_\theta(a | s)  \left\|\xi_1 - \xi_2 \right\|_1 \\
        &\stackrel{(b)}{\leq} \frac{\sqrt{d_p}C_\phi(1+\log |\mathcal{A}|)}{(1-\gamma)^2} \left\|\xi_1 - \xi_2\right\|_2
    \end{align}
    where \((a)\) uses the \(1\)-norm bound on the feature bases, and \((b)\) uses the Cauchy-Schwarz inequality. For \eqref{eqn: lipschitz smooth theta}, we use the chain rule to obtain
    \begin{align}
        \left\|\nabla_{\xi} J_{P_\xi, \tau}^{\pi_{\theta_1}} - \nabla_{\xi} J_{P_\xi, \tau}^{\pi_{\theta_2}} \right\|_2 &= \frac{1}{1-\gamma}\Bigg\| \sum_s \sum_a \Big(d_{\pi_{\theta_1}}^{P_{\xi}}(s) \pi_{\theta_1}(a | s) \sum_{s'} \phi(s, a, s') \left(r(s, a) - \tau \log \pi_{\theta_1}(a | s) + \gamma V_{P_\xi, \tau}^{\pi_{\theta_1}}(s')\right) \nonumber \\
        &- d_{\pi_{\theta_2}}^{P_{\xi}}(s) \pi_{\theta_2}(a | s) \sum_{s'} \phi(s, a, s') \left(r(s, a) - \tau \log \pi_{\theta_2}(a | s) + \gamma V_{P_\xi, \tau}^{\pi_{\theta_2}}(s') \right)\Big)\Bigg\|_2 \\
        &\leq \left| J_{P_\xi, \tau}^{\pi_{\theta_1}} - J_{P_\xi, \tau}^{\pi_{\theta_2}} \right| \|\phi(s, a, \cdot)\|_2\\
        &\quad +  \frac{\gamma}{1-\gamma} \left\| \sum_s \sum_a \left( d_{\pi_{\theta_1}}^{P_{\xi}}(s) - d_{\pi_{\theta_2}}^{P_{\xi}}(s) \right) \pi_{\theta_1}(a | s) \sum_{s'} \phi(s, a, s') V_{P_\xi, \tau}^{\pi_{\theta_1}}(s') \right\|_2 \nonumber \\
        &\quad + \frac{\gamma}{1-\gamma} \left\| \sum_s \sum_a d_{\pi_{\theta_2}}^{P_{\xi}}(s) \left( \pi_{\theta_1}(a | s) - \pi_{\theta_2}(a | s) \right) \sum_{s'} \phi(s, a, s') V_{P_\xi, \tau}^{\pi_{\theta_1}}(s') \right\|_2 \nonumber \\
        &\quad + \frac{\gamma}{1-\gamma} \left\| \sum_s \sum_a d_{\pi_{\theta_2}}^{P_{\xi}}(s) \pi_{\theta_2}(a | s) \sum_{s'} \phi(s, a, s') \left( V_{P_\xi, \tau}^{\pi_{\theta_1}}(s') - V_{P_\xi, \tau}^{\pi_{\theta_2}}(s') \right) \right\|_2 \\
        &\stackrel{(a)}{\leq}  \frac{\gamma(1 + \tau \log |\mathcal{A}|)}{(1 - \gamma)^3} C_{\phi} \left( G_1 \sqrt{|\mathcal{A}|} \|\theta_1 - \theta_2\|\right) + \frac{\gamma(1 + \tau \log |\mathcal{A}|)}{(1 - \gamma)^2} C_{\phi} (G_1 \sqrt{|\mathcal{A}|} \|\theta_1 - \theta_2\|) \nonumber\\
        &\quad + \frac{\gamma(1 + \tau \log |\mathcal{A}|)}{(1-\gamma)^3} C_\phi G_1 \sqrt{|\mathcal{A}|} \|\theta_1 - \theta_2\|_2 + \frac{2(1 + \tau \log |\mathcal{A}|)}{(1 - \gamma)^2} C_\phi G_1 \sqrt{|\mathcal{A}|} \|\theta_1 - \theta_2\|_2\\
        &\leq \frac{5 \gamma (1+ \tau \log |\mathcal{A}|)}{(1-\gamma)^3} C_\phi  G_1 \sqrt{|\mathcal{A}|}\|\theta_1 - \theta_2\|_2
    \end{align}
    where in \((a)\) we use the Lipschitz bounds on the stationary distribution, policy, and value function from Lemmas~\ref{lemma: discounted visitation lipschitz}, ~\ref{lemma: policy and entropy lipschitz}, and~\ref{lemma: discounted value lipschitz} respectively. Finally we prove~\eqref{eqn: lipschitz smooth xi} as follows:
    \begin{align}
        \left\|\nabla_{\xi} J_{P_{\xi_1}, \tau}^{\pi_{\theta}} - \nabla_{\xi} J_{P_{\xi_2}, \tau}^{\pi_{\theta}} \right\|_2 &\leq \left| J_{P_{\xi_1}, \tau}^{\pi_{\theta}} - J_{P_{\xi_2}, \tau}^{\pi_{\theta}} \right| + \frac{\gamma}{1-\gamma}\Bigg\| \sum_s \sum_a \Big(d_{\pi_{\theta}}^{P_{\xi_1}}(s) \pi_{\theta}(a | s) \sum_{s'} \phi(s, a, s') V_{P_{\xi_1}, \tau}^{\pi_{\theta}}(s') \\
        &- d_{\pi_{\theta}}^{P_{\xi_2}}(s) \pi_{\theta}(a | s) \sum_{s'} \phi(s, a, s') V_{P_{\xi_2}, \tau}^{\pi_{\theta}}(s') \Big)\Bigg\|_2 \\
        &\leq \left| J_{P_{\xi_1}, \tau}^{\pi_{\theta}} - J_{P_{\xi_2}, \tau}^{\pi_{\theta}} \right| \\
        &\quad + \frac{\gamma}{1-\gamma} \left\| \sum_s \sum_a \left( d_{\pi_{\theta}}^{P_{\xi_1}}(s) - d_{\pi_{\theta}}^{P_{\xi_2}}(s) \right) \pi_{\theta}(a | s) \sum_{s'} \phi(s, a, s') V_{P_{\xi_1}, \tau}^{\pi_{\theta}}(s') \right\|_2 \nonumber \\
        &\quad + \frac{\gamma}{1 - \gamma}\left\| \sum_s \sum_a d_{\pi_{\theta}}^{P_{\xi_2}}(s) \pi_{\theta}(a | s) \sum_{s'} \phi(s, a, s') \left( V_{P_{\xi_1}, \tau}^{\pi_{\theta}}(s') - V_{P_{\xi_2}, \tau}^{\pi_{\theta}}(s') \right) \right\|_2 \\
        &\leq \frac{\gamma}{1 - \gamma} \sum_s \left| d_{\pi_{\theta}}^{P_{\xi_1}}(s) - d_{\pi_{\theta}}^{P_{\xi_2}}(s) \right| \sum_a \pi_{\theta}(a | s) \left\| \phi(s, a, \cdot) \right\|_1 \left\| V_{P_{\xi_1}, \tau}^{\pi_{\theta_1}}(s') \right\|_\infty  \nonumber \\
        &\quad + \frac{\gamma}{1 - \gamma} \sum_s d_{\pi_{\theta}}^{P_{\xi_2}}(s) \sum_a \pi_{\theta}(a | s) \left\| \phi(s, a, \cdot) \right\|_1 \left\| V_{P_{\xi_1}, \tau}^{\pi_{\theta}}(s') - V_{P_{\xi_2}, \tau}^{\pi_{\theta}}(s')\right\|_\infty \\
        &\leq \frac{2\gamma(1+\log |\mathcal{A}|)}{(1 - \gamma)^3} C_\phi \|P_{\xi_1}(\cdot | s, a) - P_{\xi_2}(\cdot | s, a)\|_1 + \frac{\gamma^2 (1 + \tau \log |\mathcal{A}|)}{(1-\gamma)^3}  C_\phi\|P_{\xi_1} - P_{\xi_2}\|_1 \\
        &\leq \frac{3\gamma}{(1 - \gamma)^3} \sqrt{d_p}(1+\log |\mathcal{A}|) C_\phi \|\xi_1 - \xi_2\|_2
    \end{align}
\end{proof}

\begin{lemma}[Lipschitz Smoothness of the Outer Objective]
    Assume that the feature map \(\phi\) is bounded and that the regularity conditions of Lemma~\ref{lemma: lipschitz and lipschitz smoothness} hold uniformly over \(\theta \in \Theta\). Then the outer objective \(F(\xi)\) is differentiable on \(\Xi\), and its gradient is Lipschitz continuous. In particular, there exists \(L_F > 0\) such that for all \(\xi_1, \xi_2 \in \Xi\),
    \begin{align}
        \|\nabla_\xi F(\xi_1) - \nabla_\xi F(\xi_2)\|_2 \leq L_F \|\xi_1 - \xi_2\|_2.
    \end{align}
\end{lemma}

\begin{proof}
    Using Danskin's Theorem, we have
    \begin{align}
        \left\|\nabla_{\xi}F(\xi_1) - \nabla_\xi F(\xi_2)\right\| &= \left\|\nabla_{\xi} g_{P_{\xi_1}, \tau}^{\pi_{\theta_{\xi_1}}^*} - \nabla_{\xi} g_{P_{\xi_2}, \tau}^{\pi_{\theta_{\xi_2}}^*} \right\|_2 \\
        &\leq \left\|\nabla_{\xi} g_{P_{\xi_1}, \tau}^{\pi_{\theta_{\xi_1}}^*} - \nabla_{\xi} g_{P_{\xi_1}, \tau}^{\pi_{\theta_{\xi_2}}^*} \right\|_2 + \left\|\nabla_{\xi} g_{P_{\xi_1}, \tau}^{\pi_{\theta_{\xi_2}}^*} - \nabla_{\xi} g_{P_{\xi_2}, \tau}^{\pi_{\theta_{\xi_2}}^*} \right\|_2 \\
        &\leq \frac{5 \gamma (1+ \tau \log |\mathcal{A}|)}{(1-\gamma)^3} C_\phi  \left\|\pi_{\theta_{\xi_1}}^* - \pi_{\theta_{\xi_2}}^*\right\|_1 + \frac{3\gamma \sqrt{d_p}}{(1 - \gamma)^3} (1+\log |\mathcal{A}|) C_\phi \|\xi_1 - \xi_2\|_2 \\
        &\leq \frac{10 |A| \gamma (1+ \tau \log |\mathcal{A}|)^2}{\tau (1-\gamma)^5} C_\phi \|P_{\xi_1} - P_{\xi_2}\|_1 + \frac{3\gamma \sqrt{d_p}}{(1 - \gamma)^3} (1+\log |\mathcal{A}|) C_\phi \|\xi_1 - \xi_2\|_2 \\
        &\leq \frac{13 |A| \sqrt{d_p}\gamma (1+ \tau \log |\mathcal{A}|)^2}{\tau (1-\gamma)^5} C_\phi \|\xi_1 - \xi_2\|_1
    \end{align}
\end{proof}

%% file: Sections/Appendix/gradient_dominance.tex
\section{Gradient Dominance Proofs}
\label{appendix: gradient dominance}
\begin{lemma}[Gradient Dominance for Transition Kernels]
\label{lemma: gradient dominance transition kernel}
Define the distribution mismatch coefficient as
\[
    D \coloneqq \sup_{P,P'\in\mathcal{P}} \sup_{\theta \in \Theta} \left \|\frac{d_{\pi_\theta}^{P'}}{d_{\pi_\theta}^P}\right\|_{\infty}.
\]
Then for any kernel $P \in \mathcal{P}$ we have the following gradient dominance property
\[
    \min_{P'\in\mathcal{P}}\left(J_{P,\tau}^{\pi_\theta}-J_{P',\tau}^{\pi_\theta}\right) \leq \frac{D}{1 - \gamma} \max_{P'\in\mathcal{P}} (P - P')^\top \nabla_{P} J_{P,\tau}^{\pi_\theta}
\]
\end{lemma}

\begin{proof}
To achieve gradient dominance, we need to show that the maximum difference between the long term average rewards with respect to different transition kernels is upper bounded by a scaled version of the maximum difference between the transition kernels times the gradient of the long term average reward. Using Lemma~\ref{lemma: transition kernel performance difference}, for any \(P' \in \mathcal{P}\) we have
\begin{align}
    J_{P,\tau}^{\pi_\theta} - J_{P',\tau}^{\pi_\theta} &= \frac{1}{1 - \gamma}\sum_{s \in \mathcal{S}} \sum_{a \in \mathcal{A}} d_{\pi_\theta}^{P'}(s)\pi_\theta(a|s) \sum_{s' \in \mathcal{S}} \left(P(s'|s,a)-P'(s'|s,a)\right) V_{P,\tau}^{\pi_\theta}(s') \\
    &= \frac{1}{1 - \gamma} \sum_{s} \sum_a \frac{d_{\pi_\theta}^{P'}(s)}{d_{\pi_\theta}^{P}(s)} d_{\pi_\theta}^{P}(s)\pi_\theta(a|s) \sum_{s'} \left(P(s'|s,a)-P'(s'|s,a)\right) V_{P,\tau}^{\pi_\theta}(s') \\
    &\leq \frac{1}{1 - \gamma} \left( \max_{s \in \mathcal{S}} \frac{d_{\pi_\theta}^{P'}(s)}{d_{\pi_\theta}^{P}(s)} \right) \sum_{s} \sum_a \sum_{s'} d_{\pi_\theta}^{P}(s)\pi_\theta(a|s) \left(P(s'|s,a)-P'(s'|s,a)\right) V_{P,\tau}^{\pi_\theta}(s') \\
    &= \frac{1}{1 - \gamma} \left\| \frac{d_{\pi_\theta}^{P'}}{d_{\pi_\theta}^{P}} \right\|_\infty (P - P')^\top d_{\pi_\theta}^P Q_{P, \tau}^{\pi_\theta} \\
    &= \frac{1}{1 - \gamma} \left\| \frac{d_{\pi_\theta}^{P'}}{d_{\pi_\theta}^{P}} \right\|_\infty (P - P')^\top \nabla_{P} J_{P,\tau}^{\pi_\theta}
\end{align}
Taking the minimum over \(P' \in \mathcal{P}\) on the left side and noting the upper bound holds for the maximum on the right side, we obtain the gradient dominance property
\begin{align}
    \min_{P' \in \mathcal{P}} \left( J_{P,\tau}^{\pi_\theta} - J_{P',\tau}^{\pi_\theta} \right) \leq \frac{D}{1 - \gamma} \max_{P' \in \mathcal{P}} (P - P')^\top \nabla_{P} J_{P,\tau}^{\pi_\theta}
\end{align}
\end{proof}

\begin{lemma}[Gradient Dominance under Linear Parameterization]
Let the transition kernel be parameterized as \(P_\xi(s' \mid s,a) = \phi(s,a,s')^\top \xi\) for \(\xi \in \Xi\). Let \(J_{P_\xi, \tau}^{\pi_\theta}\) denote the objective under policy \(\pi_\theta\). Then the following gradient dominance property holds within the parameterized transition kernel uncertainty set:
\[
    \min_{\xi' \in \Xi} \left( J_{P_\xi, \tau}^{\pi_\theta} - J_{P_{\xi'}, \tau}^{\pi_\theta} \right) \leq \frac{D}{1 - \gamma} \max_{\xi' \in \Xi} (\xi - \xi')^\top \nabla_\xi J_{P_\xi, \tau}^{\pi_\theta},
\]
where \(D\) is the distribution mismatch coefficient.
\end{lemma}

\begin{proof}
By the performance difference lemma for transition kernels (Lemma~\ref{lemma: transition kernel performance difference}), we have
\begin{align}
    J_{P_\xi, \tau}^{\pi_\theta} - J_{P_{\xi'}, \tau}^{\pi_\theta} = \frac{1}{1 - \gamma}\sum_{s \in \mathcal{S}} \sum_{a \in \mathcal{A}} d_{\pi_\theta}^{P_{\xi'}}(s)\,\pi_\theta(a \mid s) \sum_{s' \in \mathcal{S}} \left( P_\xi(s' \mid s,a) - P_{\xi'}(s' \mid s,a) \right) V_{P_\xi, \tau}^{\pi_\theta}(s').
\end{align}

Using the linear parameterization,
\begin{align}
    P_\xi(s' \mid s,a) - P_{\xi'}(s' \mid s,a) = \phi(s,a,s')^\top (\xi - \xi'),
\end{align}
we obtain
\begin{align}
    J_{P_\xi, \tau}^{\pi_\theta} - J_{P_{\xi'}, \tau}^{\pi_\theta} &= \frac{1}{1 - \gamma}\sum_{s} \sum_a d_{\pi_\theta}^{P_{\xi'}}(s)\,\pi_\theta(a \mid s) \sum_{s'} \phi(s,a,s')^\top (\xi - \xi') V_{P_\xi, \tau}^{\pi_\theta}(s') \\
    &= \frac{1}{1 - \gamma} \sum_{s} \sum_a \frac{d_{\pi_\theta}^{P_{\xi'}}(s)}{d_{\pi_\theta}^{P_\xi}(s)} d_{\pi_\theta}^{P_\xi}(s)\,\pi_\theta(a \mid s) \sum_{s'} \phi(s,a,s')^\top (\xi - \xi') V_{P_\xi, \tau}^{\pi_\theta}(s') \\
    &\stackrel{(a)}{\leq} \frac{D}{1 - \gamma} (\xi - \xi')^\top \sum_{s}\sum_a \sum_{s'} d_{\pi_\theta}^{P_\xi}(s)\,\pi_\theta(a \mid s)\, \phi(s,a,s') V_{P_\xi, \tau}^{\pi_\theta}(s') \\
    &\stackrel{(b)}{=} \frac{D}{1 - \gamma} (\xi - \xi')^\top \nabla_\xi J_{P_\xi, \tau}^{\pi_\theta}
\end{align}
where \((a)\) uses the distribution mismatch coefficient and \((b)\) uses the formula for the gradient of the long term average reward with respect to the parameterized transition kernel. Taking the minimum over \(\xi' \in \Xi\) on the left side and noting the upper bound holds for the maximum on the right side yields the desired result:
\begin{align}
    \min_{\xi' \in \Xi} \left(J_{P_\xi, \tau}^{\pi_\theta} - J_{P_{\xi'}, \tau}^{\pi_\theta} \right) \leq \max_{\xi' \in \Xi} \frac{D}{1 - \gamma} (\xi - \xi')^\top \nabla_\xi J_{P_\xi, \tau}^{\pi_\theta}
\end{align}
\end{proof}

\begin{lemma}[Gradient Dominance for Non-rectangular Uncertainty sets]
    For a fixed policy \(\pi_\theta\), let
    \begin{align}
        \delta_{\Xi} = \min_{\xi' \in \Xi} \left[\max_{\xi_s \in \Xi_s} \xi_s^\top \nabla J_{P_{\xi'}}^{\pi_{\theta}} - \max_{\xi \in \Xi} \xi^\top \nabla J_{P_{\xi'}}^{\pi_{\theta}}\right]
    \end{align}
    be the degree of non-rectangularity in the linearly parameterized uncertainty set, where \(\Xi_s\) is the smallest s-rectangular uncertainty set within \(\Xi\). Then the following gradient domination theorem holds
    \begin{align}
        \min_{\xi' \in \Xi} \left(J_{P_{\xi}}^{\pi_{\theta}} - J_{P_{\xi'}}^{\pi_{\theta}}\right) \leq \frac{D}{1 - \gamma} \max_{\xi' \in \Xi} (\xi - \xi')^\top \nabla J_{P_{\xi}}^{\pi_{\theta}} + \frac{D}{1 - \gamma}\delta_{\Xi}
    \end{align}
    where \(D\) is the distribution mismatch coefficient.
\end{lemma}

\begin{proof}
    By Lemma~\ref{lemma: gradient dominance linear transition kernel}, we have
    \begin{align}
        \min_{\xi' \in \Xi} \left(J_{P_{\xi}}^{\pi_{\theta}} - J_{P_{\xi'}}^{\pi_{\theta}}\right) \leq \frac{D}{1 - \gamma} \max_{\xi_s \in \Xi_s} (\xi - \xi')^\top \nabla J_{P_{\xi'}}^{\pi_{\theta}}
    \end{align}
    Then by the degree of non-rectangularity, we have
    \begin{align}
        \min_{\xi' \in \Xi} \left(J_{P_{\xi}}^{\pi_{\theta}} - J_{P_{\xi'}}^{\pi_{\theta}}\right) &\leq \frac{D}{1 - \gamma} \max_{\xi' \in \Xi} (\xi - \xi')^\top \nabla J_{P_{\xi'}}^{\pi_{\theta}} + \min_{\xi' \in \Xi} \left[\max_{\xi_s \in \Xi_s} \xi_s^\top \nabla J_{P_{\xi'}}^{\pi_{\theta}} - \max_{\xi \in \Xi} \xi^\top \nabla J_{P_{\xi'}}^{\pi_{\theta}}\right] \\
        &\leq \frac{D}{1 - \gamma} \max_{\xi' \in \Xi} (\xi - \xi')^\top \nabla J_{P_{\xi'}}^{\pi_{\theta}} + \frac{D}{1 - \gamma}\delta_{\Xi}
    \end{align}
\end{proof}

%% file: Sections/Appendix/algorithms.tex
\section{Algorithm}
\label{appendix: algorithm}
\begin{theorem}[Convergence of Algorithm~\ref{alg: pgd}]
Let $\hat{\nabla} J_{P_\xi}^{\pi_{\theta}}$ be an estimator of $\nabla_\xi J_{P_\xi}^{\pi_\theta}$ satisfying \(\|\hat{\nabla} J_{P_\xi}^{\pi_\theta} - \nabla_\xi J_{P_\xi}^{\pi_\theta}\| \leq \epsilon_{\mathrm{grad}}\) and assume the policy error satisfies \(\|\theta_k - \theta_k^*\| \le \epsilon_\theta \quad \forall k\). Then, choosing \(\beta = \frac{1}{2 L_F}, \epsilon_{\mathrm{grad}} = O(\epsilon(1-\gamma)), \epsilon_\theta = O(\epsilon(1-\gamma)^4), K = O(\epsilon^{-3}(1-\gamma)^{-8})\) in Algorithm~\ref{alg: pgd} guarantees
\begin{align}
    F(\xi_{\tilde{K}}) - \min_{\xi \in \Xi} F(\xi) \le \varepsilon.
\end{align}
where \(\xi_{\tilde{K}}\) is the output of Algorithm~\ref{alg: pgd}.
\end{theorem}
\begin{proof}
    We begin with 
    \begin{align}
        \left\|\hat{\nabla} J_{P_\xi}^{\pi_{\theta_k}}\right\| &= \frac{1}{\beta} \left\|\xi_k - \beta \hat{\nabla} J_{P_\xi}^{\pi_{\theta_k}} - \xi_k\right\| \\
        &\geq \frac{1}{\beta} \left\|\xi_k - \beta \hat{\nabla} J_{P_\xi}^{\pi_{\theta_k}} - \mathrm{Proj}_{\Xi}\left(\xi_k - \beta \hat{\nabla} J_{P_\xi}^{\pi_{\theta_k}}\right)\right\| \\
        &= \|\hat{\nabla} J_{P_\xi}^{\pi_{\theta_k}} - G_k\|
    \end{align}
    where \(G_k = \frac{1}{\beta} \left(\mathrm{Proj}_{\Xi}\left(\xi_k - \beta \hat{\nabla} J_{P_\xi}^{\pi_{\theta_k}}\right) - \xi_k\right)\). This implies that \(G_k^\top \hat{\nabla} J_{P_\xi}^{\pi_{\theta_k}} \geq \frac{1}{2}\|G_k\|^2\). Then we have
    \begin{align}
        F(\xi_{k+1}) - F(\xi_k) &\geq \nabla_\xi F (\xi_k) (\xi_{k+1} - \xi_k) - \frac{L_F}{2} \|\xi_{k+1} - \xi_{k}\|^2 \\
        &= \beta G^{\top}_k \left[\nabla_\xi F(\xi_k) - \hat{\nabla}_\xi J_{P_{\xi_k}, \tau}^{\pi_{\theta_k}}\right] + \beta G^{\top}_k \hat{\nabla}_\xi J_{P_{\xi_k}, \tau}^{\pi_{\theta_k}} - \frac{L_F\beta^2}{2} \|G_k\|^2 \\
        &= \beta G^{\top}_k \left[\nabla_\xi F(\xi_k) - \nabla_\xi J_{P_{\xi_k}, \tau}^{\pi_{\theta_k}}\right] + \beta G^{\top}_k \left[\nabla_\xi J_{P_{\xi_k}, \tau}^{\pi_{\theta_k}} - \hat{\nabla}_\xi J_{P_{\xi_k}, \tau}^{\pi_{\theta_k}}\right] \\
        &\quad + \beta G^{\top}_k \hat{\nabla}_\xi J_{P_{\xi_k}, \tau}^{\pi_{\theta_k}} - \frac{L_F\beta^2}{2} \|G_k\|^2 \\
        &\geq -\beta l_\pi \|G_k\| \|\theta_k^* - \theta_k\|_2 - \beta \epsilon_{\mathrm{grad}} \|G_k\| + \beta G^{\top}_k \hat{\nabla}_\xi J_{P_{\xi_k}, \tau}^{\pi_{\theta_k}} - \frac{L_F\beta^2}{2} \|G_k\|^2 \\
        &\geq -\beta l_\pi \epsilon_\theta \|G_k\| - \beta \epsilon_{\mathrm{grad}} \|G_k\| + \frac{\beta}{2} \|G_k\|^2 - \frac{\beta}{4} \|G_k\|^2 \\
        &\geq \frac{\beta \|G_k\|^2}{8} - 2\beta (l_\pi \epsilon_\theta + \epsilon_{\mathrm{grad}})^2
    \end{align}
    We can then use a telescoping sum to bound the minimum magnitude of \(G_k\),
    \begin{align}
        \|G_{\tilde{K}}\| &= \min_{0, 1, \dots, K} \|G_k\| \\
        &\leq \sqrt{\frac{1}{K} \sum_{k=0}^{K-1} \|G_k\|^2} \\
        &\leq \sqrt{\frac{1}{K} \sum_{k=0}^{K-1} \left( \frac{8 (F(\xi_{k+1}) - F(\xi_k))}{\beta} + 16\left(l_\pi \epsilon_\theta + \epsilon_{\mathrm{grad}}\right)^2\right)} \\
        &\leq \sqrt{\left( \frac{8 (F(\xi_{K}) - F(\xi_0))}{K\beta} + 16\left(l_\pi \epsilon_\theta + \epsilon_{\mathrm{grad}}\right)^2\right)} \\
        &\leq \sqrt{\left(\frac{16 (1 + \tau \log |\mathcal{A}|)}{K\beta} + 16\left(l_\pi \epsilon_\theta + \epsilon_{\mathrm{grad}}\right)^2\right)} \\
        &\leq \sqrt{\frac{16 (1 + \tau \log |\mathcal{A}|)}{K\beta}} + 4\left(l_\pi \epsilon_\theta + \epsilon_{\mathrm{grad}}\right)
    \end{align}
    Finally, we can bound the suboptimality gap at timestep \(\tilde{K}\) as
    \begin{align}
        F(\xi_{\tilde{K}}) - \min_{\xi \in \Xi} F(\xi) &\leq \frac{D}{1 - \gamma} \max_{\xi \in \Xi} (\xi_{\tilde{K}} - \xi)^\top \nabla F(\xi_{\tilde{K}}) \\
        &\leq \frac{D}{1 - \gamma} \max_{\xi \in \Xi} (\xi_{\tilde{K}} - \xi)^\top \hat{\nabla} J_{P_{\xi_{\tilde{K}}}, \tau}^{\pi_{\theta_{\tilde{K}}}} + \max_{\xi \in \Xi} \frac{D}{1 - \gamma} \left\|\xi_{\tilde{K}} - \xi\right\| \left\|\nabla F(\xi_{\tilde{K}}) - \hat{\nabla} J_{P_{\xi_{\tilde{K}}}, \tau}^{\pi_{\theta_{\tilde{K}}}} \right\| \\
        &\leq \frac{D}{1 - \gamma} \max_{\xi \in \Xi} (\xi_{\tilde{K}} - \xi)^\top \hat{\nabla} J_{P_{\xi_{\tilde{K}}}, \tau}^{\pi_{\theta_{\tilde{K}}}} + \frac{D}{1 - \gamma} D_{\Xi} (l_\pi \epsilon_\theta + \epsilon_{\mathrm{grad}}) \\
        &\leq \max_{\xi \in \Xi} \frac{D}{\beta (1 - \gamma)} (\xi_{\tilde{K}+1} - \xi)^\top \left(\xi_{\tilde{K}} + \beta \hat{\nabla} J_{P_{\xi_{\tilde{K}}}, \tau}^{\pi_{\theta_{\tilde{K}}}} - \xi_{\tilde{K}+1}\right) + \frac{D}{\beta (1-\gamma)} (\xi_{\tilde{K} + 1} - \xi)^\top (\xi_{\tilde{K}} - \xi_{\tilde{K} + 1}) \nonumber \\
        &\quad + \frac{D}{1 - \gamma}\beta \|G_k\| (l_\xi + \epsilon_{\mathrm{grad}}) + \frac{D}{1 - \gamma} D_{\Xi} (l_\pi \epsilon_\theta + \epsilon_{\mathrm{grad}}) \\
        &\leq \max_{\xi \in \Xi} \frac{D}{\beta (1-\gamma)} \|\xi_{\tilde{K}+1} - \xi\| \|\beta G_{\tilde{K}}\| + \frac{D}{1 - \gamma}\beta (l_\xi + \epsilon_{\mathrm{grad}}) \|G_{\tilde{K}}\| \\
        &\quad + \frac{D}{1 - \gamma} D_{\Xi} (l_\pi \epsilon_\theta + \epsilon_{\mathrm{grad}}) \\
        &\leq \left(\frac{D}{1 - \gamma} D_{\Xi} + \frac{D}{1 - \gamma} \beta(l_\xi + \epsilon_{\mathrm{grad}}) \right) \left(\sqrt{\frac{16 (1 + \tau \log |\mathcal{A}|)}{K\beta}} + 4\left(l_\pi \epsilon_\theta + \epsilon_{\mathrm{grad}}\right)\right) \\
        &\quad + \frac{D}{1 - \gamma} D_{\Xi} (l_\pi \epsilon_\theta + \epsilon_{\mathrm{grad}})
    \end{align}
    We need this value to be less than \(\epsilon\). Choosing \(K = \mathcal{O}\left(\epsilon^{-3}(1-\gamma)^{-8}\right),\, \epsilon_{\mathrm{grad}} = \mathcal{O}(\epsilon(1-\gamma)),\, \epsilon_\theta = \mathcal{O}(\epsilon(1-\gamma)^4)\) yields \(F(\xi_{\tilde{K}}) - \min_{\xi \in \Xi} F(\xi) \leq \epsilon\).
\end{proof}

\begin{theorem}[Convergence of Algorithm~\ref{alg: frank_wolfe}]
Let $\hat{\nabla} F(\xi)$ be an estimator of the gradient $\nabla F(\xi)$ satisfying $\|\hat{\nabla} F(\xi) - \nabla F(\xi)\| \leq l_\pi \epsilon_\theta + \epsilon_{\mathrm{grad}}$. Assume the objective $F$ has a finite bounded curvature constant $C_f$ \citep{lacoste2016convergence} over the domain $\Xi$ with diameter $D_{\Xi}$. Then, choosing $K = \mathcal{O}(C_f^2\epsilon^{-2}(1-\gamma)^{-2})$, $\epsilon_{\mathrm{grad}} = O(\epsilon(1-\gamma))$, and $\epsilon_\theta = O(\epsilon(1-\gamma)^4)$ guarantees that Algorithm~\ref{alg: frank_wolfe} yields:
\begin{align}
    F(\xi_{\tilde{K}}) - \min_{\xi \in \Xi} F(\xi) \leq \epsilon + \frac{D\delta_{\Xi}}{1 - \gamma},
\end{align}
where $\delta_{\Xi}$ is the irreducible error due to the non-rectangularity of $\Xi$ and \(\xi_{\tilde{K}}\) is the output of Algorithm~\ref{alg: frank_wolfe}.
\end{theorem}

\begin{proof}
We begin with the standard Frank-Wolfe update $\xi_{k+1} = \xi_k + \zeta_k d_k$, where $d_k = s_k - \xi_k$ and $s_k = \arg\min_{s \in \Xi} \langle s, \hat{\nabla} F(\xi_k) \rangle$. We define \(\delta = l_\pi \epsilon_\theta + \epsilon_{\mathrm{grad}}\). By the definition of the curvature constant $C_f$:
\begin{align}
    F(\xi_\gamma) &\leq F(\xi_k) + \gamma \langle \nabla F(\xi_k), s_k - \xi_k \rangle + \frac{\gamma^2}{2} C_f, \quad \forall \gamma \in [0, 1], \\
    C_f &= \sup_{\substack{\xi, s \in \Xi, \zeta \in [0, 1] \\ \tau \in (0, \infty), y = \xi + \zeta(s-\xi)}} \frac{2}{\zeta^2} \left(F(y) - F(\xi) - \left\langle \nabla_\xi F(\xi), y - \xi \right\rangle\right)
\end{align}
Let $\hat{g}_k := \langle \xi_k - s_k, \hat{\nabla} F(\xi_k) \rangle$ be the \textit{noisy} Frank-Wolfe gap and $J_k := \langle \xi_k - s_k, \nabla F(\xi_k) \rangle$ be the \textit{true} gap. We relate them using the gradient noise bound:
\begin{align}
    J_k = \langle \xi_k - s_k, \nabla F(\xi_k) \rangle &= \langle \xi_k - s_k, \hat{\nabla} F(\xi_k) \rangle + \langle \xi_k - s_k, \nabla F(\xi_k) - \hat{\nabla} F(\xi_k) \rangle \\
    &\geq \hat{g}_k - \|\xi_k - s_k\| \|\nabla F(\xi_k) - \hat{\nabla} F(\xi_k)\| \\
    &\geq \hat{g}_k - D_\Xi \delta
\end{align}
Substituting this into the descent lemma:
\begin{equation}
    F(\xi_\gamma) \leq F(\xi_k) - \gamma (\hat{g}_k - D_\Xi \delta) + \frac{\gamma^2}{2} C_f
\end{equation}
The line search finds $\xi_{k+1}$ such that $F(\xi_{k+1}) \leq \min_{\gamma \in [0,1]} F(\xi_\gamma)$. Minimizing the quadratic upper bound with respect to $\gamma$ yields the optimal step size $\gamma^* = \min\{1, \frac{\hat{g}_k - D_\Xi \delta}{C_f}\}$. Plugging $\gamma^*$ into the bound:
\begin{equation}
    F(\xi_{k+1}) \leq F(\xi_k) - \min \left\{ \frac{(\hat{g}_k - D_\Xi \delta)^2}{2C_f}, \hat{g}_k - D_\Xi \delta - \frac{C_f}{2} \right\}
\end{equation}
To simplify, let $g'_k = \max\{0, \hat{g}_k - D_\Xi \delta\}$. Summing from $k=0$ to $K - 1$ and letting $h_k = F(\xi_k) - F(\xi^*)$:
\begin{equation}
    h_{K} \leq h_0 - \sum_{k=0}^{K-1} \max \left\{ \frac{(g'_k)^2}{2C_f}, g'_k - \frac{C_f}{2} \right\}
\end{equation}
Since $h_{0} \geq h_{K} \geq 0$, we have:
\begin{equation}
    (K+1) \max \left\{ \frac{(\tilde{g}'_K)^2}{2C_f}, \tilde{g}'_K - \frac{C_f}{2} \right\} \leq h_0
\end{equation}
Dividing by $K$ and rearranging yields:
\begin{equation}
    \max \left\{ \frac{(\tilde{g}'_K)^2}{2C_f}, \tilde{g}'_K - \frac{C_f}{2} \right\} \leq \frac{h_0}{K}
\end{equation}

\textbf{Case 1: $\tilde{g}'_K \leq C_f$.}
The first term in the minimum dominates:
\begin{equation}
    \frac{(\tilde{g}'_K)^2}{2C_f} \leq \frac{h_0}{K} \implies \tilde{g}'_K \leq \sqrt{\frac{2C_f h_0}{K}}
\end{equation}

\textbf{Case 2: $\tilde{g}'_K > C_f$.}
The second term dominates:
\begin{equation}
    \tilde{g}'_K - \frac{C_f}{2} \leq \frac{h_0}{K} \implies \tilde{g}'_K \leq \frac{h_0}{K} + \frac{C_f}{2} \leq \frac{2 h_0}{\sqrt{K}}
\end{equation}

Finally, substituting $\tilde{g}'_K = \tilde{g}_K - D_\Xi \delta$ yields:
\begin{equation}
    \tilde{g}_K \leq \frac{\max\{2h_0, C_f\}}{\sqrt{K}} + D_\Xi \delta
\end{equation}
where $\delta = l_\pi \epsilon_\theta + \epsilon_{\mathrm{grad}}$. Then by Lemma~\ref{lemma: non-rectangular gradient dominance}, we have
\begin{align}
    F(\xi_{\tilde{K}}) - \min_{\xi' \in \Xi} F(\xi') &\leq \frac{D}{1-\gamma} \left[\delta_{\Xi} + \tilde{g}_K \right]\\
    &\leq \frac{D}{1-\gamma} \left[\delta_{\Xi} + \frac{\max\{2h_0, C_f\}}{\sqrt{K}} + D_\Xi \delta\right]
\end{align}
Choosing \(K = O(C_f^2\epsilon^{-2}(1-\gamma)^{-2}), \epsilon_\theta = O(\epsilon(1-\gamma)^4), \epsilon_{\mathrm{grad}} = O(\epsilon(1-\gamma))\) yields a solution that is \(\left(\epsilon + \frac{D \delta_\Xi}{1 - \gamma}\right)\)-optimal.
\end{proof}

%% file: Sections/Appendix/stochastic_gradient_estimation.tex
\section{Stochastic Gradient Estimation}
\label{appendix: stochastic gradient estimation}
\begin{lemma}[\citep{ganesh2025a}, Properties of the MLMC Estimator]
\label{lemma: mlmc_properties}
Consider a time-homogeneous, ergodic Markov chain $(Z_t)_{t\geq 0}$ with a unique invariant distribution $d_Z$ and a mixing time $t_{\mathrm{mix}}$. Assume that $\nabla F(x, Z)$ is an estimate of the gradient $\nabla F(x)$. Let the bias and variance be bounded such that:
\begin{align*}
    \| \mathbb{E}_{d_Z}[\nabla F(x, Z)] - \nabla F(x) \|^2 &\leq \delta^2, \\
    \| \nabla F(x, Z_t) - \mathbb{E}_{d_Z}[\nabla F(x, Z)] \|^2 &\leq \sigma^2 \quad \forall t \geq 0.
\end{align*}
If $Q \sim \mathrm{Geom}(1/2)$, then the following MLMC estimator
\begin{equation}
    g_{\mathrm{MLMC}} = g^0 + \begin{cases} 
        2^Q (g^Q - g^{Q-1}), & \text{if } 2^Q \leq T_{\max} \\
        0, & \text{otherwise}
    \end{cases}
\end{equation}
where $g^j = 2^{-j} \sum_{t = 0}^{2^j - 1} \nabla F(x, Z_t)$, satisfies the following inequalities:
\begin{enumerate}
    \item[(a)] $\mathbb{E}[g_{\mathrm{MLMC}}] = \mathbb{E}[g^{\lfloor \log T_{\max} \rfloor}]$
    \item[(b)] $\mathbb{E}[\| \nabla F(x) - g_{\mathrm{MLMC}} \|^2] \leq \mathcal{O} \left( \sigma^2 t_{\mathrm{mix}} \log_2 T_{\max} + \delta^2 \right)$
    \item[(c)] $\| \nabla F(x) - \mathbb{E}[g_{\mathrm{MLMC}}] \|^2 \leq \mathcal{O} \left( \sigma^2 t_{\mathrm{mix}} T_{\max}^{-1} + \delta^2 \right)$
\end{enumerate}
\end{lemma}

\begin{lemma}[Sample Complexity of MLMC Gradient Estimator]
Suppose the per-transition gradient estimator $\nabla_\xi F(\xi,Z_t)$ has variance bounded by $\sigma^2$, and the value-function approximation satisfies  
\[
\mathbb{E}_{s\sim d_{\pi_{\theta^*}}^{P_\xi}}\left[\left(\hat V_{P_\xi, \tau}^{\pi_{\theta^*}} - V_{P_\xi, \tau}^{\pi_{\theta^*}}\right)^2\right]\le\epsilon_v^2 .
\]  
Let $t_{\mathrm{mix}}$ be the mixing time, and $p_{\min} = \min_{s,a,s'}P_\xi(s'\mid s,a) > 0$. Set  
\[
C_{\mathrm{sig}} = \mathcal{O}\!\left( \frac{C_\phi^2(1+\log|\mathcal{A}|)^2}{(1-\gamma)^4 p_{\min}} \right),\quad
\epsilon_v = \frac{(1-\gamma)\epsilon}{2\gamma C_\phi\sqrt{p_{\min}}},\quad
T_{\max}= \big\lceil C_{\mathrm{sig}} \, t_{\mathrm{mix}} \, \epsilon^{-2} \big\rceil, \quad N = \left\lceil C_{\mathrm{sig}} \, t_{\mathrm{mix}} \, \frac{\log_2(T_{\max})}{\epsilon^2} \, \log\frac{1}{\beta} \right\rceil
\]  
Take $N$ independent MLMC gradient samples. Partition them into $K=\lceil 8\log(1/\beta)\rceil$ blocks, and let $\nabla_\xi F_{\mathrm{MoM}}(\xi)$ be their geometric median of block means. Then with probability at least $1-\beta$,  
\[
    \big\| \nabla_\xi F_{\mathrm{MoM}}(\xi) - \nabla_\xi F(\xi) \big\|_2^2 \le \epsilon^2
\]
in \(\mathbb{E}\left[T_{\mathrm{total}}\right] = \mathcal{O}\left( \frac{\log^2\left(\epsilon^{-2}(1 - \gamma)^{-4}\right)}{\epsilon^2(1 - \gamma)^4} \log\frac{1}{\beta}\right)\) samples.
\end{lemma}

\begin{proof}
    We first define our estimator of the gradient for a single transition \(Z_t = (s_t, a_t, s_{t+1})\) as
    \[
    \nabla_\xi F(\xi, Z_t) := \frac{1}{1 - \gamma} 
    \left[\frac{\phi(s_t, a_t, s_{t+1})}{P_\xi(s_{t+1} \mid s_t, a_t)} \left(r_\tau(s_t, a_t) - \gamma \hat{V}_{P_\xi, \tau}^{\pi_{\theta^*}}(s_{t+1})\right)\right],
    \]
    where \(\hat{V}_{P_\xi, \tau}^{\pi_{\theta^*}}\) is an estimate of the value function. For our MLMC estimation, we assume
    \[
    \left\|\hat{V}_{P_\xi, \tau}^{\pi_{\theta^*}} - V_{P_\xi, \tau}^{\pi_{\theta^*}}\right\|_{L^2(d_{\pi_{\theta^*}}^{P_\xi})} := 
    \mathbb{E}_{s \sim d_{\pi_{\theta^*}}^{P_\xi}} \left[\left(\hat{V}_{P_\xi, \tau}^{\pi_{\theta^*}}(s) - V_{P_\xi, \tau}^{\pi_{\theta^*}}(s)\right)^2\right] \leq \epsilon_v^2.
    \]
    We bound the bias of the single-step gradient estimator. Let $\Delta V(s') = V_{P_\xi, \tau}^{\pi_{\theta^*}}(s') - \hat{V}_{P_\xi, \tau}^{\pi_{\theta^*}}(s')$.
    \begin{align}
        \left\| \mathbb{E}_{Z \sim d_Z}[\nabla_\xi F(\xi, Z)] - \nabla_\xi F(\xi) \right\|_2^2
        &= \frac{\gamma^2}{(1 - \gamma)^2} \left[ \left\| \mathbb{E}_{(s, a, s') \sim d_Z} \left[\frac{\phi(s, a, s')}{P_\xi(s' | s, a)} \Delta V(s') \right] \right\|_2^2 \right] \\
        &\leq \frac{\gamma^2}{(1 - \gamma)^2} \mathbb{E}_{s \sim d_{\pi_{\theta^*}}^{P_\xi}, a \sim \pi(\cdot | s))} \left[ \left\| \mathbb{E}_{s' \sim P_\xi(\cdot \mid s, a)} \left[\frac{\phi(s, a, s')}{P_\xi(s' | s, a)} \Delta V(s') \right] \right\|_2^2 \right] \\
        &\leq \frac{\gamma^2}{(1 - \gamma)^2} \mathbb{E}_{s \sim d_{\pi_{\theta^*}}^{P_\xi}, a \sim \pi(\cdot | s))} \left[ \sum_{s'} \frac{\|\phi(s, a, s')\|_2^2}{P_\xi(s' \mid s, a)} \mathbb{E}_{s' \sim P_\xi(\cdot \mid s, a)} \left[ (\Delta V(s'))^2 \right] \right] \\
        &\leq \frac{\gamma^2}{(1 - \gamma)^2} \mathbb{E}_{s \sim d_{\pi_{\theta^*}}^{P_\xi}, a \sim \pi(\cdot | s))} \left[ \frac{C_\phi^2}{p_{\min}} \mathbb{E}_{s' \sim P_\xi(\cdot \mid s, a)} \left[ (\Delta V(s'))^2 \right] \right] \\
        &= \frac{\gamma^2 C_\phi^2}{(1 - \gamma)^2 p_{\min}} \mathbb{E}_{s' \sim d_{\pi_{\theta^*}}^{P_\xi}} \left[ \left(V_{P_\xi, \tau}^{\pi_{\theta^*}}(s') - \hat{V}_{P_\xi, \tau}^{\pi_{\theta^*}}(s')\right)^2 \right] \\
        &\leq \frac{\gamma^2 C_\phi^2 \epsilon_v^2}{(1 - \gamma)^2 p_{\min}}
    \end{align}
    Similarly, the variance of the single-sample estimator is
    \begin{align}
    \left\| \nabla_\xi F(\xi, Z_t) - \mathbb{E}_{Z \sim d_Z}[\nabla_\xi F(\xi, Z)] \right\|_2^2
    &= \frac{1}{(1-\gamma)^2} \Bigg\|\frac{\phi(s_t, a_t, s_{t+1})}{P_\xi(s_{t+1} \mid s_t, a_t)} \left(r_\tau(s_t, a_t) - \gamma \hat{V}_{P_\xi, \tau}^{\pi_{\theta^*}}(s_{t+1})\right) \\
    &\quad - \mathbb{E}_{(s, a, s') \sim d_Z} \left[\frac{\phi(s, a, s')}{P_\xi(s' \mid s, a)} \left(r_\tau(s, a) - \gamma \hat{V}_{P_\xi, \tau}^{\pi_{\theta^*}}(s')\right)\right]\Bigg\|_2^2 \\
    &\leq \frac{1}{(1-\gamma)^2} \mathbb{E}_{(s, a, s') \sim d_Z} \left[\left\|\frac{\phi(s, a, s')}{P_\xi(s' \mid s, a)} \left(r_\tau(s, a) - \gamma \hat{V}_{P_\xi, \tau}^{\pi_{\theta^*}}(s')\right)\right\|^2_2\right] \\
    &\leq \frac{2}{(1-\gamma)^2} \mathbb{E}_{(s, a, s')\sim d_Z} \left[\frac{\|\phi(s, a, s')\|_2^2}{P_\xi(s' \mid s, a)^2} r_\tau(s, a)^2\right] \\
    &\quad + \frac{2\gamma^2}{(1-\gamma)^2} \mathbb{E}_{(s, a, s')\sim d_Z} \left[\frac{\|\phi(s, a, s')\|_2^2}{P_\xi(s' \mid s, a)^2} \hat{V}_{P_\xi, \tau}^{\pi_{\theta^*}}(s')^2\right] \\
    &\leq \frac{2}{(1 - \gamma)^2} \frac{C_\phi^2}{p_\mathrm{min}} \left((1 + \log |\mathcal{A}|)^2 + \frac{(1 + \log |\mathcal{A}|)^2}{(1 - \gamma)^2}\right) \\
    &\leq \frac{4 C_\phi^2 (1 + \log |\mathcal{A}|)^2}{(1 - \gamma)^4 p_\mathrm{min}}
    \end{align}
    Thus, by Lemma~\ref{lemma: mlmc_properties}, we have
    \begin{align}
        \mathbb{E}\left[\|\nabla_\xi F(\xi) - \nabla_\xi F_{\mathrm{MLMC}}(\xi)\|_2^2\right] &\leq 
        \mathcal{O} \left(\sigma^2 t_{\mathrm{mix}} \log_2 T_{\max} + \delta^2\right) \\
        &= \mathcal{O}\left(\frac{4 C_\phi^2 (1 + \log |\mathcal{A}|)^2}{(1 - \gamma)^4 p_\mathrm{min}} t_{\mathrm{mix}} \log_2 T_{\max} + \frac{\gamma^2 C_\phi^2 \epsilon_v^2}{(1 - \gamma)^2 p_{\min}}\right) \\
        \left\|\nabla_\xi F(\xi) - \mathbb{E}\left[\nabla_\xi F_{\mathrm{MLMC}}\right]\right\|^2 &\leq \mathcal{O}\left(\sigma^2 t_{\mathrm{mix}} T_{\max}^{-1} + \delta^2\right) \\
        &= \mathcal{O}\left(\frac{4 C_\phi^2 (1 + \tau \log |\mathcal{A}|)^2}{(1 - \gamma)^4 p_{\mathrm{min}}} t_{\mathrm{mix}} T_{\mathrm{max}}^{-1}+ \frac{\gamma^2 C_\phi^2 \epsilon_v^2}{(1 - \gamma)^2 p_{\mathrm{min}}}\right)
    \end{align}
    However, we require a bound in the form \(\|\nabla_\xi F_{\mathrm{est}}(\xi) - \nabla_\xi F(\xi)\|^2 \leq \epsilon^2\). Thus, we construct a Median-of-Means estimator \(\nabla_\xi F_{\mathrm{MoM}}(\xi)\) using i.i.d copies of \(\nabla_\xi F_{\mathrm{MLMC}}(\xi)\). We have
    \begin{align}
        \|\nabla_\xi F(\xi) - \nabla_\xi F_{\mathrm{MoM}}(\xi)\|^2 &\leq 2\|\mathbb{E}\left[\nabla_\xi F_{\mathrm{MLMC}}\right] - \nabla_\xi F_{\mathrm{MoM}}(\xi)\|^2 + 2\left\|\nabla_\xi F(\xi) - \mathbb{E}\left[\nabla_\xi F_{\mathrm{MLMC}}\right]\right\|^2 \\
        &\leq 2\|\mathbb{E}\left[\nabla_\xi F_{\mathrm{MLMC}}\right] - \nabla_\xi F_{\mathrm{MoM}}(\xi)\|^2 \\
        &\quad + \mathcal{O}\left(\frac{4 C_\phi^2 (1 + \tau \log |\mathcal{A}|)^2}{(1 - \gamma)^4 p_{\mathrm{min}}} t_{\mathrm{mix}} T_{\mathrm{max}}^{-1}+ \frac{\gamma^2 C_\phi^2 \epsilon_v^2}{(1 - \gamma)^2 p_{\mathrm{min}}}\right)
    \end{align}
    Let $X^{(1)}, \dots, X^{(N)}$ be $N$ independent copies of $\nabla_\xi F_{\mathrm{MLMC}}(\xi)$. We partition these into $K$ blocks of size $m = N/K$ and compute the empirical mean for each block $k \in \{1, \dots, K\}$: $\bar{g}_k = \frac{1}{m} \sum_{i \in G_k} X^{(i)}$. The robust estimator is defined as the geometric median: $\nabla_\xi F_{\mathrm{MoM}}(\xi) = \arg\min_{y} \sum_{k=1}^K \| y - \bar{g}_k \|_2$.
    To use the MoM concentration bound, we first bound the variance of a single MLMC sample, $\Sigma_{\mathrm{MLMC}} = \mathrm{Tr}(\mathrm{Cov}(\nabla_\xi F_{\mathrm{MLMC}}))$. Let $J = \lfloor \log_2 T_{\max} \rfloor$. We rewrite the estimator using indicator variables:
    \begin{align}
        \nabla_\xi F_{\mathrm{MLMC}} = g^0 + \sum_{j=1}^{J} \mathbb{I}_{\{Q=j\}} \Delta_j, \quad \text{where } \Delta_j = 2^j(g^j - g^{j-1}).
    \end{align}
    Using the Law of Total Variance and the disjointness of $\{Q=j\}$, we have:
    \begin{align}
        \mathrm{Tr}(\mathrm{Cov}(\nabla_\xi F_{\mathrm{MLMC}})) &\leq \mathbb{E}_Q \left[ \mathbb{E} [ \| \nabla_\xi F_{\mathrm{MLMC}} \|^2 \mid Q ] \right] \\
        &= P(Q > J) \mathbb{E}[\|g^0\|^2] + \sum_{j=1}^J P(Q=j) \mathbb{E}[\|g^0 + \Delta_j\|^2] \\
        &\leq 2^{-J} \sigma^2 + \sum_{j=1}^J 2^{-j} (2 \sigma^2 + 2 \mathbb{E}[\|\Delta_j\|^2])
    \end{align}
    From the ergodicity of the Markov chain, $\mathbb{E}[\|\Delta_j\|^2] = 2^{2j} \mathbb{E}[\|g^j - g^{j-1}\|^2] \leq C \sigma^2 t_{\mathrm{mix}} 2^j$. Substituting this:
    \begin{align}
        \mathrm{Tr}(\mathrm{Cov}(\nabla_\xi F_{\mathrm{MLMC}})) &\leq 2 \sigma^2 + \sum_{j=1}^J 2^{-j} (2 C \sigma^2 t_{\mathrm{mix}} 2^j) \\
        &\leq 2 \sigma^2 + 2 C \sigma^2 t_{\mathrm{mix}} J \\
        &= \mathcal{O}(\sigma^2 t_{\mathrm{mix}} \log_2 T_{\max}).
    \end{align}
    By the concentration properties of the Geometric Median-of-Means \citep{minsker2015geometric}, for a failure probability $\beta$, choosing $K = \lceil 8 \log(1/\beta) \rceil$ yields:
    \begin{align}
        \|\nabla_\xi F_{\mathrm{MoM}}(\xi) - \mathbb{E}\left[\nabla_\xi F_{\mathrm{MLMC}}\right]\|^2 &\leq \mathcal{O}\left( \frac{\mathrm{Tr}(\mathrm{Cov}(\nabla_\xi F_{\mathrm{MLMC}}))}{N} \log(1/\beta) \right) \\
        &\leq \mathcal{O}\left( \frac{\sigma^2 t_{\mathrm{mix}} \log_2 T_{\max}}{N} \log(1/\beta) \right).
    \end{align}
    To ensure that this value is less than \(\mathcal{O}(\epsilon^2)\) with probability \(1 - \beta\), we set
    \begin{align}
        T_\mathrm{max} &= \mathcal{O}\left(\frac{C_\phi^2 (1 + \tau \log |\mathcal{A}|)^2 t_{\mathrm{mix}}}{(1 - \gamma)^4 p_{\mathrm{min}} \epsilon^2}\right) = \mathcal{O}(\epsilon^{-2}(1 - \gamma)^{-4}) \\
        N &= \mathcal{O}\left(\frac{C_\phi^2 (1 + \tau \log |\mathcal{A}|)^2 t_{\mathrm{mix}}}{(1 - \gamma)^4 p_{\mathrm{min}} \epsilon^2} \log \left(\frac{C_\phi^2 (1 + \tau \log |\mathcal{A}|)^2 t_{\mathrm{mix}}}{(1 - \gamma)^4 p_{\mathrm{min}} \epsilon^2}\right) \log\frac{1}{\beta}\right) = \mathcal{O}\left( \frac{\log\left(\epsilon^{-2}(1 - \gamma)^{-4}\right)}{\epsilon^2(1 - \gamma)^4} \log\frac{1}{\beta}\right)
    \end{align}
    We also require
    \begin{align}
        \epsilon_v = \mathcal{O}\left(\frac{(1 - \gamma)\epsilon}{2 \gamma C_\phi \sqrt{p_\mathrm{min}}}\right)
    \end{align}
    to ensure the bias from the individual estimate does not dominate the bound, which requires \(\mathcal{O}\left(\epsilon^{-2}(1-\gamma)^{-4}\right)\) samples using either linear \citep{li2024high} or neural \citep{ke2024improved} temporal difference learning. Since each MLMC estimate takes on average \(\mathcal{O}(\log_2 T_{\max})\) samples, we have
    \begin{align}
        \mathbb{E}\left[T_{\mathrm{total}}\right] = \mathcal{O}\left(N \log_2 T_{\mathrm{max}}\right) = \mathcal{O}\left( \frac{\log^2\left(\epsilon^{-2}(1 - \gamma)^{-4}\right)}{\epsilon^2(1 - \gamma)^4} \log\frac{1}{\beta}\right)
    \end{align}
\end{proof}

\begin{algorithm}[htbp]
\caption{MLMC Gradient Estimation}
\label{alg: mlmc_gradient_discounted}
\begin{algorithmic}[1]
\REQUIRE Parameter $x$, error tolerance $\epsilon$, failure probability $\beta$, mixing time $t_{\mathrm{mix}}$
\ENSURE Robust gradient estimate $\widehat{g}_{\mathrm{robust}}$
\STATE $C_\mathrm{sig} \gets \frac{4\,C_\phi^2\,(1+\log|\mathcal{A}|)^2}{(1-\gamma)^4\,p_{\min}}$
\STATE $\epsilon_v \gets \frac{(1-\gamma)\epsilon}{2\,\gamma\,C_\phi\,\sqrt{p_{\min}}}$
\STATE $T_{\max} \gets \left\lceil C_\mathrm{sig}\,t_{\mathrm{mix}}\,\epsilon^{-2} \right\rceil$
\STATE $K \gets \lceil 8\,\log(1/\beta) \rceil$
\STATE $N \gets \left\lceil C_\mathrm{sig}\,t_{\mathrm{mix}}\,\frac{\log_2(T_{\max})}{\epsilon^2}\,\log(1/\beta)\right\rceil$
\STATE $m \gets \lfloor N/K \rfloor$
\FOR{$i = 1$ to $N$}
    \STATE Sample $Q^{(i)} \sim \mathrm{Geom}(1/2)$
    \STATE Run Markov chain $(Z_t^{(i)})$ starting at $x$ for $2^{Q^{(i)}}$ steps
    \STATE Compute single‑step gradient:
    \[
        g^{0,(i)} \gets \frac{1}{1-\gamma}
        \frac{\phi(s_0^{(i)},a_0^{(i)},s_1^{(i)})}{P_\xi(s_1^{(i)}\mid s_0^{(i)},a_0^{(i)})}
        \Big(r_\tau(s_0^{(i)},a_0^{(i)}) - \gamma\,\hat{V}_{P_\xi,\tau}^{\pi_{\theta^*}}(s_1^{(i)})\Big)
    \]
    \IF{$2^{Q^{(i)}} \le T_{\max}$}
        \STATE Compute level $Q^{(i)}$ average:
        \[
            g^{Q^{(i)},(i)} \gets 2^{-Q^{(i)}}\sum_{t=0}^{2^{Q^{(i)}}-1}
                \frac{1}{1-\gamma}
                \frac{\phi(s_t^{(i)},a_t^{(i)},s_{t+1}^{(i)})}{P_\xi(s_{t+1}^{(i)}\mid s_t^{(i)},a_t^{(i)})}
                \Big(r_\tau(s_t^{(i)},a_t^{(i)}) - \gamma\,\hat{V}_{P_\xi,\tau}^{\pi_{\theta^*}}(s_{t+1}^{(i)})\Big)
        \]
        \STATE Compute level $Q^{(i)}-1$ average:
        \[
            g^{Q^{(i)}-1,(i)} \gets 2^{-(Q^{(i)}-1)}\sum_{t=0}^{2^{Q^{(i)}-1}-1}
                \frac{1}{1-\gamma}
                \frac{\phi(s_t^{(i)},a_t^{(i)},s_{t+1}^{(i)})}{P_\xi(s_{t+1}^{(i)}\mid s_t^{(i)},a_t^{(i)})}
                \Big(r_\tau(s_t^{(i)},a_t^{(i)}) - \gamma\,\hat{V}_{P_\xi,\tau}^{\pi_{\theta^*}}(s_{t+1}^{(i)})\Big)
        \]
        \STATE $X^{(i)} \gets g^{0,(i)} + 2^{Q^{(i)}}\,(g^{Q^{(i)},(i)} - g^{Q^{(i)}-1,(i)})$
    \ELSE
        \STATE $X^{(i)} \gets g^{0,(i)}$
    \ENDIF
\ENDFOR
\FOR{$k = 1$ to $K$}
    \STATE $\bar{g}_k \gets \frac{1}{m}\sum_{i=(k-1)m+1}^{k m} X^{(i)}$
\ENDFOR
\STATE $\widehat{g}_{\mathrm{robust}} \gets \arg\min_y \sum_{k=1}^K \|y - \bar{g}_k\|_2$
\STATE \textbf{Return } $\widehat{g}_{\mathrm{robust}}$
\end{algorithmic}
\end{algorithm}

%% file: Sections/Appendix/theoretical_analysis.tex
\section{Global Convergence Guarantees}
\label{appendix: global convergence guarantees}
We first note that both of the algorithms used for the policy oracle provide sample complexity guarantees in terms of the error in the objective \(J_{P_\xi, \tau}^{\pi_\theta}\), \(\epsilon_\mathrm{obj}\), instead of the policy parameter. Since both algorithms require \(\epsilon_\theta = \mathcal{O}(\epsilon(1-\gamma)^{4})\) and by Lemma~\ref{lemma: lipschitz and lipschitz smoothness}, \(\epsilon_{\mathrm{obj}} \leq L_\pi \epsilon_\theta\), the sample complexity for these algorithms is \(\mathcal{O}(\epsilon^{-2}(1-\gamma)^{-10})\) and \(\mathcal{O}(\epsilon^{-3}(1-\gamma)^{-8.5})\) respectively. Thus, we have the following global optimality guarantees.

\begin{theorem}[Discounted Robust MDP Sample Complexity]
Let \(\epsilon > 0\) and let \(\pi_{\theta_{\tilde{K}}}\) and \(\xi_{\tilde{K}}\) be the outputs from either Algorithm~\ref{alg: pgd} or Algorithm~\ref{alg: frank_wolfe}. Choose
\(
\tau = \Theta \big(\epsilon(1-\gamma)\big), \quad
\epsilon_{\mathrm{model}} = \Theta\big(\epsilon(1-\gamma)\big), \quad
\epsilon_{\mathrm{policy}} = \Theta\big(\epsilon(1-\gamma)^4\big), \quad
\epsilon_{\mathrm{grad}} = \Theta\big(\epsilon(1-\gamma)\big).
\)
Then the discounted robust MDP solution \(\xi_{\tilde K}\) satisfies
\[
J_{P_{\xi_{\tilde{K}}}}^{\pi_{\theta_{\tilde{K}}}} - \min_{\xi \in \Xi}  \max_{\theta \in \Theta} J_{P_\xi}^{\pi_\theta} \le \epsilon.
\]
Moreover, the total sample complexity satisfies
\[
T_{\mathrm{total}} =
\begin{cases}
\mathcal{O}\big(\epsilon^{-5}(1-\gamma)^{-18}\big), & \text{Algorithm}~\ref{alg: pgd}, \\
\mathcal{O}\big(C_f^2 \epsilon^{-4}(1-\gamma)^{-12}\big), & \text{Algorithm}~\ref{alg: frank_wolfe}.
\end{cases}
\]
\end{theorem}

\begin{proof}
    We denote the total error for the algorithm as \(\epsilon\), where Lemma~\ref{lemma: entropy regularized policy optimality} allows us to decompose it as
    \begin{align}
        \epsilon = 2 \epsilon_{\mathrm{outer}} + \frac{\tau \log |\mathcal{A}|}{1 - \gamma} + \frac{2 L_V \epsilon_{\mathrm{model}}}{1 - \gamma}
    \end{align}
    where \(\epsilon_\mathrm{outer}\) is the error from the entropy regularized robust MDP algorithm. Thus, we require \(\epsilon_{\mathrm{outer}} = \frac{\epsilon}{6}, \tau = \frac{\epsilon (1 - \gamma)}{3\log |\mathcal{A}|}, \epsilon_{\mathrm{model}} = \frac{\epsilon (1 - \gamma)}{6 L_V}\). We can further break down \(\epsilon_{\mathrm{outer}}\) as
    \begin{align}
        \epsilon_\mathrm{outer} = \epsilon_{\mathrm{alg}} + \mathcal{O}(l_\pi \epsilon_\mathrm{policy} + \epsilon_\mathrm{grad}). 
    \end{align}
    where \(\epsilon_\mathrm{alg}\) is the error coming directly from the algorithm without gradient or policy error. With Theorem~\ref{theorem: pgd_convergence}, this requires an iteration complexity of \(K = \mathcal{O}(\epsilon^{-3}(1-\gamma)^{-8})\), and with Theorem~\ref{theorem: fw_convergence} requires \(K = \mathcal{O}\left(C_f^2\epsilon^{-2}(1-\gamma)^{-2}\right)\). Since we require \(\epsilon_{\mathrm{\theta}} = \mathcal{O}(\epsilon(1-\gamma)^4)\) and the policy oracle gives the sample complexity in terms of the objective error, we have a policy oracle sample complexity \(T = \mathcal{O}(\epsilon^{-2}(1 - \gamma)^{10})\) when using Algorithm 1 in \citet{mondal2024improved}. For the MLMC gradient estimator, we require \(\epsilon_\mathrm{grad} = \mathcal{O}(\epsilon(1 - \gamma))\), which requires a sample complexity of \(N = \mathcal{O}(\epsilon^{-2}(1 - \gamma)^{-6})\). Thus, in the discounted setting using PGD (Algorithm~\ref{alg: pgd}), we have 
    \begin{align}
        T_{\mathrm{total}} = K(T + N) = \mathcal{O}(\epsilon^{-3}(1 - \gamma)^{-8}) \left(\mathcal{O}(\epsilon^{-2}(1 - \gamma)^{-10}) + \mathcal{O}(\epsilon^{-2}(1 - \gamma)^{-6})\right) = \mathcal{O}(\epsilon^{-5}(1 - \gamma)^{-18}),
    \end{align}
    and when using the Frank Wolfe Algorithm (Algorithm~\ref{alg: frank_wolfe}), we have
    \begin{align}
        T_{\mathrm{total}} = K(T + N) = \mathcal{O}(C_f^2\epsilon^{-2}(1 - \gamma)^{-2}) \left(\mathcal{O}(\epsilon^{-2}(1 - \gamma)^{-10}) + \mathcal{O}(\epsilon^{-2}(1 - \gamma)^{-6})\right) = \mathcal{O}(C_f^2\epsilon^{-4}(1 - \gamma)^{-12})
    \end{align}
\end{proof}

\begin{corollary}[Average-Reward Robust MDP Sample Complexity]
Let \(\epsilon > 0\) under the same setting as Theorem~\ref{theorem: discounted_sample_complexity}. Consider the average-reward robust MDP obtained via the reduction technique with discount factor \(\gamma = 1 - \Theta\left(\frac{\epsilon}{H}\right)\) in \citep{wang2022near}. Using the Frank--Wolfe algorithm with the MLMC gradient estimator and SRVR-PG as the policy oracle \citep{liu2020improved}, the returned solution \(\xi_{\tilde K}\) satisfies
\[
g_{P_{\xi_{\tilde{K}}}}^{\pi_{\theta_{\tilde{K}}}} - \min_{\xi \in \Xi}  \max_{\theta \in \Theta} g_{P_\xi}^{\pi_\theta} \le \epsilon.
\]
Moreover, the total sample complexity is
\[
T_{\mathrm{total}} = \mathcal{O}\left(C_f^2 \epsilon^{-10.5} H^{5.5}\right).
\]
\end{corollary}

\begin{proof}
    Using the Frank Wolfe algorithm as in Theorem~\ref{theorem: discounted_sample_complexity} with the MLMC gradient estimator, we have \(K = \mathcal{O}\left(C_f^2\epsilon^{-2}(1-\gamma)^{-2}\right)\) and \(N = \mathcal{O}(\epsilon^{-2}(1 - \gamma)^{-6})\). For the average reward setting, we use SRVR-PG in \citet{liu2020improved} as the policy oracle, which requires \(T = \mathcal{O}((1-\gamma)^{-8.5}\epsilon^{-3})\) environment interactions to achieve the required policy error. Thus, the total sample complexity is
    \begin{align}
        T_\mathrm{total} = \mathcal{O}\left(C_f^2\epsilon^{-2}(1-\gamma)^{-2}\right) \left(\mathcal{O}\left(\epsilon^{-2}(1 - \gamma)^{-4}\right) + \mathcal{O}\left((1-\gamma)^{-8.5}\epsilon^{-3}\right)\right) = \mathcal{O}\left(C_f^2 \epsilon^{-5}(1-\gamma)^{-10.5}\right)
    \end{align}
    Applying the reduction technique in \citet{wang2022near} with \(\gamma = 1 - \Theta\left(\frac{\epsilon}{H}\right)\) and \(\epsilon_\gamma = \mathcal{O}\left(\frac{\epsilon}{1 - \gamma}\right)\) yields
    \begin{align}
        T_\mathrm{total} = \mathcal{O}\left(C_f^2 \epsilon_\gamma^{-5}(1 - \gamma)^{-10.5}\right) = \mathcal{O}\left(C_f^2 \epsilon^{-5}(1 - \gamma)^{-5.5}\right) = \mathcal{O}\left(C_f^2 \epsilon^{-10.5} H^{5.5}\right)
    \end{align}
\end{proof}

\paragraph{Estimating the Span H.} The average-reward reduction requires knowledge of the span $H = \max_{s,s'} (V^{\pi^*}(s) - V^{\pi^*}(s'))$ of the optimal policy's bias function. In many domains, $H$ can be bounded using structural properties: in queueing systems it relates to maximum queue lengths, in inventory control to the range of inventory levels, and in episodic tasks to episode length. When such bounds are unavailable, $H$ can be estimated from sample trajectories by computing empirical bias estimates $\hat{V}(s)$ using temporal-difference or Monte Carlo methods with a preliminary policy, then setting $\hat{H} = \max_{s,s'} |\hat{V}(s) - \hat{V}(s')|$ \citep{pesquerel2022regret, agrawal2024optimistic}.

Alternatively, one can employ an adaptive approach: start with a conservative estimate $\hat{H}_0$, apply the reduction with $\gamma = 1 - \Theta(\epsilon/\hat{H}_0)$, and if the resulting policy has empirical span exceeding $\hat{H}_0$, increase the estimate and restart. While overestimating $H$ increases sample complexity (our bounds scale as $O(H^{5.5})$), underestimating only affects approximation quality, making conservative estimates preferable. For problems where tight $H$ bounds are impractical, our discounted results (Theorem~\ref{theorem: discounted_sample_complexity}) apply directly without this parameter.

%% file: Sections/Appendix/supporting_lemmas.tex
\section{Supporting Lemmas}
\label{appendix: supporting lemmas}
\begin{lemma}[Policy and Entropy Lipschitz bounds]
\label{lemma: policy and entropy lipschitz}
If a policy $\pi_\theta(a \mid s)$ satisfies
\[
\sup_{s,a,\theta} \|\nabla_\theta \log \pi_\theta(a \mid s)\|_2 \le L,
\]
then for all $s$ and all $\theta_1,\theta_2$,
\begin{align}
\|\pi_{\theta_1}(\cdot \mid s) - \pi_{\theta_2}(\cdot \mid s)\|_1
&\leq L \sqrt{|\mathcal{A}|}\|\theta_1-\theta_2\|, \label{eqn: pi lipschitz}\\
\|\pi_{\theta_1}(\cdot \mid s)\log \pi_{\theta_1}(\cdot \mid s)
- \pi_{\theta_2}(\cdot \mid s)\log \pi_{\theta_2}(\cdot \mid s)\|_1
&\leq L \sqrt{|\mathcal{A}|}\|\theta_1-\theta_2\|. \label{eqn: lipschitz pi log pi}
\end{align}
\end{lemma}

\begin{proof}
    We first prove \eqref{eqn: pi lipschitz} using the following pointwise bound:
    \begin{align}
        |\pi_{\theta_1}(a | s) - \pi_{\theta_2}(a | s)| &= \left|e^{\log \pi_{\theta_1}(a | s)} - e^{\log \pi_{\theta_2}(a | s)} \right| \\
        &\leq \left|\log \pi_{\theta_1}(a | s) - \log \pi_{\theta_2}(a | s)\right| \sup_{0 \leq \pi_\theta \leq 1} e^{\log \pi_\theta} \\
        &\leq \left|\log \pi_{\theta_1}(a | s) - \log \pi_{\theta_2}(a | s)\right|
    \end{align}
    Then applying this bound to all actions yields
    \begin{align}
        \|\pi_{\theta_1}(\cdot | s) - \pi_{\theta_2}(\cdot | s)\|_1 &= \sum_{a} |\pi_{\theta_1}(a | s) - \pi_{\theta_2}(a | s)| \\
        &\leq \sum_{a} \left|\log \pi_{\theta_1}(a | s) - \log \pi_{\theta_2}(a | s)\right| \\
        &\leq \sqrt{|\mathcal{A}|} \left\|\log \pi_{\theta_1}(a | s) - \log \pi_{\theta_2}(a | s)\right\|_2 \\
        &\stackrel{(a)}{\leq} L\sqrt{|\mathcal{A}|} \left\|\theta_1 - \theta_2\right\|_2
    \end{align}
    where \((a)\) uses the Mean Value Theorem with the bounded gradient assumption to obtain the Lipschitz constant. For \eqref{eqn: lipschitz pi log pi}, we have
    \begin{align}
        |\pi_{\theta_1}(a | s) \log \pi_{\theta_1}(a | s) - \pi_{\theta_2}(a | s) \log \pi_{\theta_2}(a | s)| &\leq \left|e^{\log \pi_{\theta_1}(a | s)} \log \pi_{\theta_1}(a | s) - e^{\log \pi_{\theta_2}(a | s)} \log \pi_{\theta_2}(a | s)\right| \\
        &\leq |\log \pi_{\theta_1}(a | s) - \log \pi_{\theta_2}(a | s)| \sup_{0 \leq \pi_\theta \leq 1} \left|(1+\log \pi_{\theta}) e^{\log \pi_{\theta}}\right| \\
        &\leq |\log \pi_{\theta_1}(a | s) - \log \pi_{\theta_2}(a | s)|
    \end{align}
    Then 
    \begin{align}
        \|\pi_{\theta_1}(a | s) \log \pi_{\theta_1}(a | s) - \pi_{\theta_2}(a | s) \log \pi_{\theta_2}(a | s)\|_1 &= \sum_{a} |\pi_{\theta_1}(a | s) \log \pi_{\theta_1}(a | s) - \pi_{\theta_2}(a | s) \log \pi_{\theta_2}(a | s)| \\
        &\leq \sum_{a} \left|\log \pi_{\theta_1}(a | s) - \log \pi_{\theta_2}(a | s)\right| \\
        &\leq \sqrt{|\mathcal{A}|} \left\|\log \pi_{\theta_1}(a | s) - \log \pi_{\theta_2}(a | s)\right\|_2 \\
        &\leq L\sqrt{|\mathcal{A}|} \left\|\theta_1 - \theta_2\right\|_2
    \end{align}
\end{proof}

%% file: Sections/Appendix/limitations_and_future_work.tex
\section{Limitations and Future Work}

While our work provides the first end-to-end sample complexity guarantees for s-rectangular and non-rectangular robust MDPs with general policy parameterization, several important limitations warrant discussion. Our current analysis requires the ergodicity assumption (Assumption~\ref{assumption: ergodicity}), which is primarily invoked in our MLMC gradient estimator to establish mixing time bounds. A natural extension would be to relax this to the weaker unichain assumption, where uniform mixing does not hold. A concurrent work by \citet{satheesh2026regretanalysisunichainaverage} has developed MLMC-based analysis for unichain MDPs in the non-robust setting, and extending their techniques to the robust setting would allow such bounds to hold. Additionally, our Frank-Wolfe algorithm for non-rectangular uncertainty sets (Algorithm~\ref{alg: frank_wolfe}) achieves an $(\epsilon + \frac{D\delta_\Xi}{1-\gamma})$-optimal solution, where $\delta_\Xi$ represents an irreducible approximation error due to non-rectangularity (Lemma~\ref{lemma: non-rectangular gradient dominance}). This error is unavoidable as strong duality fails for non-rectangular uncertainty sets \citep{grand2023beyond}, so the minimax problem $\max_\pi \min_{P \in \mathcal{P}} g^\pi_P$ does not generally admit a saddle point. However, to enable gradient-based optimization over general policy parameterizations, we must require strong duality to reverse the optimization order to $\min_{P \in \mathcal{P}} \max_\pi g^\pi_P$, which allows direct policy updates without projection. Thus, the only approach is to introduces the gap $\delta_\Xi$ between the s-rectangular relaxation and the true non-rectangular problem.

Perhaps most critically, while we achieve the first polynomial sample complexity guarantees in the average-reward setting, the resulting bounds are admittedly large: $O(C_f^2 \epsilon^{-10.5} H^{5.5})$ for the Frank-Wolfe algorithm. This stems primarily from the reduction technique that converts the average-reward problem to a discounted problem with $\gamma = 1 - \Theta(\epsilon/H)$ \citep{wang2022near}, where the effective horizon $\frac{1}{1-\gamma} = \Theta(H/\epsilon)$ grows inversely with target accuracy, and our MLMC gradient estimator requires $O(\epsilon^{-2}(1-\gamma)^{-4})$ samples while the policy oracle requires $O(\epsilon^{3}(1-\gamma)^{-8.5})$ iterations. Developing more sample-efficient approaches represents a critical direction for future work, including: (i) direct average-reward analysis that avoids reduction to the discounted setting, (ii) improved reduction techniques with better dependence on $H$ and $\epsilon$, and (iii) adaptive discount factor schemes that reduce samples needed in early iterations. Despite these limitations, our results establish the first rigorous foundation for this problem class, and the gap between our current bounds and information-theoretic lower bounds remains an interesting open question.

%% file: example_paper.bib
@article{sun2024policy,
  title={Policy optimization for robust average reward mdps},
  author={Sun, Zhongchang and He, Sihong and Miao, Fei and Zou, Shaofeng},
  journal={Advances in Neural Information Processing Systems},
  volume={37},
  pages={17348--17372},
  year={2024}
}

@article{xu2025efficient,
  title={Efficient $ Q $-Learning and Actor-Critic Methods for Robust Average Reward Reinforcement Learning},
  author={Xu, Yang and Ganesh, Swetha and Aggarwal, Vaneet},
  journal={arXiv preprint arXiv:2506.07040},
  year={2025}
}

@article{goyal2023robust,
  title={Robust markov decision processes: Beyond rectangularity},
  author={Goyal, Vineet and Grand-Clement, Julien},
  journal={Mathematics of Operations Research},
  volume={48},
  number={1},
  pages={203--226},
  year={2023},
  publisher={INFORMS}
}

@article{iyengar2005robust,
  title={Robust dynamic programming},
  author={Iyengar, Garud N},
  journal={Mathematics of Operations Research},
  volume={30},
  number={2},
  pages={257--280},
  year={2005},
  publisher={INFORMS}
}

@article{wiesemann2013robust,
  title={Robust Markov decision processes},
  author={Wiesemann, Wolfram and Kuhn, Daniel and Rustem, Ber{\c{c}}},
  journal={Mathematics of Operations Research},
  volume={38},
  number={1},
  pages={153--183},
  year={2013},
  publisher={INFORMS}
}

@article{nilim2003robustness,
  title={Robustness in Markov decision problems with uncertain transition matrices},
  author={Nilim, Arnab and Ghaoui, Laurent},
  journal={Advances in neural information processing systems},
  volume={16},
  year={2003}
}

@InProceedings{pmlr-v235-chen24s,
  title = 	 {Accelerated Policy Gradient for s-rectangular Robust {MDP}s with Large State Spaces},
  author =       {Chen, Ziyi and Huang, Heng},
  booktitle = 	 {Proceedings of the 41st International Conference on Machine Learning},
  pages = 	 {6847--6880},
  year = 	 {2024},
  editor = 	 {Salakhutdinov, Ruslan and Kolter, Zico and Heller, Katherine and Weller, Adrian and Oliver, Nuria and Scarlett, Jonathan and Berkenkamp, Felix},
  volume = 	 {235},
  series = 	 {Proceedings of Machine Learning Research},
  month = 	 {21--27 Jul},
  publisher =    {PMLR},
  url = 	 {https://proceedings.mlr.press/v235/chen24s.html}
}

@article{li2023policy,
  title={Policy gradient algorithms for robust mdps with non-rectangular uncertainty sets},
  author={Li, Mengmeng and Kuhn, Daniel and Sutter, Tobias},
  journal={arXiv preprint arXiv:2305.19004},
  year={2023}
}

@inproceedings{wang2025provable,
  title={Provable Policy Gradient for Robust Average-Reward MDPs Beyond Rectangularity},
  author={Wang, Qiuhao and Zha, Yuqi and Ho, Chin Pang and Petrik, Marek},
  booktitle={Forty-second International Conference on Machine Learning},
  year={2025}
}

@article{liu2025linear,
  title={Linear Mixture Distributionally Robust Markov Decision Processes},
  author={Liu, Zhishuai and Xu, Pan},
  journal={arXiv preprint arXiv:2505.18044},
  year={2025}
}

@article{mai2021robust,
  title={Robust entropy-regularized markov decision processes},
  author={Mai, Tien and Jaillet, Patrick},
  journal={arXiv preprint arXiv:2112.15364},
  year={2021}
}

@article{agarwal2021theory,
  title={On the theory of policy gradient methods: Optimality, approximation, and distribution shift},
  author={Agarwal, Alekh and Kakade, Sham M and Lee, Jason D and Mahajan, Gaurav},
  journal={Journal of Machine Learning Research},
  volume={22},
  number={98},
  pages={1--76},
  year={2021}
}

@article{leonardos2021global,
  title={Global convergence of multi-agent policy gradient in markov potential games},
  author={Leonardos, Stefanos and Overman, Will and Panageas, Ioannis and Piliouras, Georgios},
  journal={arXiv preprint arXiv:2106.01969},
  year={2021}
}

@InProceedings{pmlr-v80-papini18a,
  title = 	 {Stochastic Variance-Reduced Policy Gradient},
  author =       {Papini, Matteo and Binaghi, Damiano and Canonaco, Giuseppe and Pirotta, Matteo and Restelli, Marcello},
  booktitle = 	 {Proceedings of the 35th International Conference on Machine Learning},
  pages = 	 {4026--4035},
  year = 	 {2018},
  editor = 	 {Dy, Jennifer and Krause, Andreas},
  volume = 	 {80},
  series = 	 {Proceedings of Machine Learning Research},
  month = 	 {10--15 Jul},
  publisher =    {PMLR},
  url = 	 {https://proceedings.mlr.press/v80/papini18a.html}
}

@article{liu2020improved,
  title={An improved analysis of (variance-reduced) policy gradient and natural policy gradient methods},
  author={Liu, Yanli and Zhang, Kaiqing and Basar, Tamer and Yin, Wotao},
  journal={Advances in Neural Information Processing Systems},
  volume={33},
  pages={7624--7636},
  year={2020}
}

@article{xu2019sample,
  title={Sample efficient policy gradient methods with recursive variance reduction},
  author={Xu, Pan and Gao, Felicia and Gu, Quanquan},
  journal={arXiv preprint arXiv:1909.08610},
  year={2019}
}

@inproceedings{fatkhullin2023stochastic,
  title={Stochastic policy gradient methods: Improved sample complexity for fisher-non-degenerate policies},
  author={Fatkhullin, Ilyas and Barakat, Anas and Kireeva, Anastasia and He, Niao},
  booktitle={International Conference on Machine Learning},
  pages={9827--9869},
  year={2023},
  organization={PMLR}
}

@inproceedings{
ganesh2025a,
title={A Sharper Global Convergence Analysis for Average Reward Reinforcement Learning via an Actor-Critic Approach},
author={Swetha Ganesh and Washim Uddin Mondal and Vaneet Aggarwal},
booktitle={Forty-second International Conference on Machine Learning},
year={2025},
url={https://openreview.net/forum?id=stfnyxnhAm}
}

@article{blanchet2019unbiased,
  title={Unbiased multilevel Monte Carlo: Stochastic optimization, steady-state simulation, quantiles, and other applications},
  author={Blanchet, Jose H and Glynn, Peter W and Pei, Yanan},
  journal={arXiv preprint arXiv:1904.09929},
  year={2019}
}

@inproceedings{mondal2024improved,
  title={Improved sample complexity analysis of natural policy gradient algorithm with general parameterization for infinite horizon discounted reward markov decision processes},
  author={Mondal, Washim U and Aggarwal, Vaneet},
  booktitle={International Conference on Artificial Intelligence and Statistics},
  pages={3097--3105},
  year={2024},
  organization={PMLR}
}

@article{wang2022near,
  title={Near sample-optimal reduction-based policy learning for average reward mdp},
  author={Wang, Jinghan and Wang, Mengdi and Yang, Lin F},
  journal={arXiv preprint arXiv:2212.00603},
  year={2022}
}

@article{grand2023beyond,
  title={Beyond discounted returns: Robust markov decision processes with average and blackwell optimality},
  author={Grand-Clement, Julien and Petrik, Marek and Vieille, Nicolas},
  journal={arXiv preprint arXiv:2312.03618},
  year={2023}
}

@article{wang2025bellman,
  title={Bellman Optimality of Average-Reward Robust Markov Decision Processes with a Constant Gain},
  author={Wang, Shengbo and Si, Nian},
  journal={arXiv preprint arXiv:2509.14203},
  year={2025}
}

@inproceedings{asadi2018lipschitz,
  title={Lipschitz continuity in model-based reinforcement learning},
  author={Asadi, Kavosh and Misra, Dipendra and Littman, Michael},
  booktitle={International conference on machine learning},
  pages={264--273},
  year={2018},
  organization={PMLR}
}

@article{lacoste2016convergence,
  title={Convergence rate of frank-wolfe for non-convex objectives},
  author={Lacoste-Julien, Simon},
  journal={arXiv preprint arXiv:1607.00345},
  year={2016}
}

@inproceedings{peng2018sim,
  title={Sim-to-real transfer of robotic control with dynamics randomization},
  author={Peng, Xue Bin and Andrychowicz, Marcin and Zaremba, Wojciech and Abbeel, Pieter},
  booktitle={2018 IEEE international conference on robotics and automation (ICRA)},
  pages={3803--3810},
  year={2018},
  organization={IEEE}
}

@inproceedings{wang2023robust,
  title={Robust average-reward markov decision processes},
  author={Wang, Yue and Velasquez, Alvaro and Atia, George and Prater-Bennette, Ashley and Zou, Shaofeng},
  booktitle={Proceedings of the AAAI Conference on Artificial Intelligence},
  volume={37},
  number={12},
  pages={15215--15223},
  year={2023}
}

@InProceedings{pmlr-v202-wang23am,
  title = 	 {Model-Free Robust Average-Reward Reinforcement Learning},
  author =       {Wang, Yue and Velasquez, Alvaro and Atia, George K. and Prater-Bennette, Ashley and Zou, Shaofeng},
  booktitle = 	 {Proceedings of the 40th International Conference on Machine Learning},
  pages = 	 {36431--36469},
  year = 	 {2023},
  editor = 	 {Krause, Andreas and Brunskill, Emma and Cho, Kyunghyun and Engelhardt, Barbara and Sabato, Sivan and Scarlett, Jonathan},
  volume = 	 {202},
  series = 	 {Proceedings of Machine Learning Research},
  month = 	 {23--29 Jul},
  publisher =    {PMLR},
  url = 	 {https://proceedings.mlr.press/v202/wang23am.html}
}

@article{chen2025sample,
  title={Sample Complexity of Distributionally Robust Average-Reward Reinforcement Learning},
  author={Chen, Zijun and Wang, Shengbo and Si, Nian},
  journal={arXiv preprint arXiv:2505.10007},
  year={2025}
}

@article{kumar2023policy,
  title={Policy gradient for rectangular robust markov decision processes},
  author={Kumar, Navdeep and Derman, Esther and Geist, Matthieu and Levy, Kfir Y and Mannor, Shie},
  journal={Advances in Neural Information Processing Systems},
  volume={36},
  pages={59477--59501},
  year={2023}
}

@article{li2022first,
  title={First-order policy optimization for robust markov decision process},
  author={Li, Yan and Lan, Guanghui and Zhao, Tuo},
  journal={arXiv preprint arXiv:2209.10579},
  year={2022}
}

@article{li2023first,
  title={First-order policy optimization for robust policy evaluation},
  author={Li, Yan and Lan, Guanghui},
  journal={arXiv preprint arXiv:2307.15890},
  year={2023}
}

@article{kumar2022efficient,
  title={Efficient policy iteration for robust markov decision processes via regularization},
  author={Kumar, Navdeep and Levy, Kfir and Wang, Kaixin and Mannor, Shie},
  journal={arXiv preprint arXiv:2205.14327},
  year={2022}
}

@inproceedings{
kumar2023towards,
title={Towards Faster Global Convergence of Robust Policy Gradient Methods},
author={Navdeep Kumar and Ilnura Usmanova and Kfir Yehuda Levy and Shie Mannor},
booktitle={Sixteenth European Workshop on Reinforcement Learning},
year={2023},
url={https://openreview.net/forum?id=cWrwdbEBx5}
}

@InProceedings{pmlr-v202-wang23i,
  title = 	 {Policy Gradient in Robust {MDP}s with Global Convergence Guarantee},
  author =       {Wang, Qiuhao and Ho, Chin Pang and Petrik, Marek},
  booktitle = 	 {Proceedings of the 40th International Conference on Machine Learning},
  pages = 	 {35763--35797},
  year = 	 {2023},
  editor = 	 {Krause, Andreas and Brunskill, Emma and Cho, Kyunghyun and Engelhardt, Barbara and Sabato, Sivan and Scarlett, Jonathan},
  volume = 	 {202},
  series = 	 {Proceedings of Machine Learning Research},
  month = 	 {23--29 Jul},
  publisher =    {PMLR},
  url = 	 {https://proceedings.mlr.press/v202/wang23i.html}
}

@article{ghosh2025scaling,
  title={Scaling Online Distributionally Robust Reinforcement Learning: Sample-Efficient Guarantees with General Function Approximation},
  author={Ghosh, Debamita and Atia, George K and Wang, Yue},
  journal={arXiv preprint arXiv:2512.18957},
  year={2025}
}

@article{minsker2015geometric,
  title={Geometric median and robust estimation in Banach spaces},
  author={Minsker, Stanislav},
  year={2015}
}

@misc{jaxrl,
  author = {Kostrikov, Ilya},
  doi = {10.5281/zenodo.5535154},
  month = {10},
  title = {{JAXRL: Implementations of Reinforcement Learning algorithms in JAX}},
  url = {https://github.com/ikostrikov/jaxrl},
  year = {2021}
}

@article{lu2022discovered,
    title={Discovered policy optimisation},
    author={Lu, Chris and Kuba, Jakub and Letcher, Alistair and Metz, Luke and Schroeder de Witt, Christian and Foerster, Jakob},
    journal={Advances in Neural Information Processing Systems},
    volume={35},
    pages={16455--16468},
    year={2022}
}

@inproceedings{jaksch2010near,
 author = {Auer, Peter and Jaksch, Thomas and Ortner, Ronald},
 booktitle = {Advances in Neural Information Processing Systems},
 editor = {D. Koller and D. Schuurmans and Y. Bengio and L. Bottou},
 pages = {},
 publisher = {Curran Associates, Inc.},
 title = {Near-optimal Regret Bounds for Reinforcement Learning},
 url = {https://proceedings.neurips.cc/paper_files/paper/2008/file/e4a6222cdb5b34375400904f03d8e6a5-Paper.pdf},
 volume = {21},
 year = {2008}
}

@inproceedings{shipra2017optimistic,
 author = {Agrawal, Shipra and Jia, Randy},
 booktitle = {Advances in Neural Information Processing Systems},
 editor = {I. Guyon and U. Von Luxburg and S. Bengio and H. Wallach and R. Fergus and S. Vishwanathan and R. Garnett},
 pages = {},
 publisher = {Curran Associates, Inc.},
 title = {Optimistic posterior sampling for reinforcement learning: worst-case regret bounds},
 url = {https://proceedings.neurips.cc/paper_files/paper/2017/file/3621f1454cacf995530ea53652ddf8fb-Paper.pdf},
 volume = {30},
 year = {2017}
}

@article{agrawal2024optimistic,
  title={Optimistic Q-learning for average reward and episodic reinforcement learning},
  author={Agrawal, Priyank and Agrawal, Shipra},
  journal={arXiv preprint arXiv:2407.13743},
  year={2024}
}

@inproceedings{ganesh2025order,
  title={Order-Optimal Regret with Novel Policy Gradient Approaches in Infinite-Horizon Average Reward MDPs},
  author={Ganesh, Swetha and Mondal, Washim Uddin and Aggarwal, Vaneet},
  booktitle={International Conference on Artificial Intelligence and Statistics},
  pages={3421--3429},
  year={2025},
  organization={PMLR}
}

@inproceedings{bai2024regret,
  title={Regret analysis of policy gradient algorithm for infinite horizon average reward markov decision processes},
  author={Bai, Qinbo and Mondal, Washim Uddin and Aggarwal, Vaneet},
  booktitle={Proceedings of the AAAI Conference on Artificial Intelligence},
  volume={38},
  number={10},
  pages={10980--10988},
  year={2024}
}

@article{gaur2024closing,
  title={Closing the gap: Achieving global convergence (last iterate) of actor-critic under markovian sampling with neural network parametrization},
  author={Gaur, Mudit and Bedi, Amrit Singh and Wang, Di and Aggarwal, Vaneet},
  journal={arXiv preprint arXiv:2405.01843},
  year={2024}
}

@inproceedings{suttle2023beyond,
  title={Beyond exponentially fast mixing in average-reward reinforcement learning via multi-level monte carlo actor-critic},
  author={Suttle, Wesley A and Bedi, Amrit and Patel, Bhrij and Sadler, Brian M and Koppel, Alec and Manocha, Dinesh},
  booktitle={International Conference on Machine Learning},
  pages={33240--33267},
  year={2023},
  organization={PMLR}
}

@inproceedings{wang2024non,
  title={Non-asymptotic analysis for single-loop (natural) actor-critic with compatible function approximation},
  author={Wang, Yudan and Wang, Yue and Zhou, Yi and Zou, Shaofeng},
  booktitle={Proceedings of the 41st International Conference on Machine Learning},
  pages={51771--51824},
  year={2024}
}

@article{patel2024towards,
  title={Towards global optimality for practical average reward reinforcement learning without mixing time oracles},
  author={Patel, Bhrij and Suttle, Wesley A and Koppel, Alec and Aggarwal, Vaneet and Sadler, Brian M and Bedi, Amrit Singh and Manocha, Dinesh},
  journal={arXiv preprint arXiv:2403.11925},
  year={2024}
}

@inproceedings{ganesh2025orderneural,
  title={Order-optimal global convergence for actor-critic with general policy and neural critic parametrization},
  author={Ganesh, Swetha and Chen, Jiayu and Mondal, Washim Uddin and Aggarwal, Vaneet},
  booktitle={The 41st Conference on Uncertainty in Artificial Intelligence},
  year={2025}
}

@article{li2024high,
  title={High-probability sample complexities for policy evaluation with linear function approximation},
  author={Li, Gen and Wu, Weichen and Chi, Yuejie and Ma, Cong and Rinaldo, Alessandro and Wei, Yuting},
  journal={IEEE transactions on information theory},
  volume={70},
  number={8},
  pages={5969--5999},
  year={2024},
  publisher={IEEE}
}

@inproceedings{ke2024improved,
  title={An improved finite-time analysis of temporal difference learning with deep neural networks},
  author={Ke, Zhifa and Wen, Zaiwen and Zhang, Junyu},
  booktitle={Proceedings of the 41st International Conference on Machine Learning},
  pages={23407--23429},
  year={2024}
}

@InProceedings{gong2020duality,
  title = 	 {A Duality Approach for Regret Minimization in Average-Award Ergodic Markov Decision Processes},
  author =       {Gong, Hao and Wang, Mengdi},
  booktitle = 	 {Proceedings of the 2nd Conference on Learning for Dynamics and Control},
  pages = 	 {862--883},
  year = 	 {2020},
  editor = 	 {Bayen, Alexandre M. and Jadbabaie, Ali and Pappas, George and Parrilo, Pablo A. and Recht, Benjamin and Tomlin, Claire and Zeilinger, Melanie},
  volume = 	 {120},
  series = 	 {Proceedings of Machine Learning Research},
  month = 	 {10--11 Jun},
  publisher =    {PMLR},
  url = 	 {https://proceedings.mlr.press/v120/gong20a.html}
}

@inproceedings{pesquerel2022regret,
 author = {Pesquerel, Fabien and Maillard, Odalric-Ambrym},
 booktitle = {Advances in Neural Information Processing Systems},
 editor = {S. Koyejo and S. Mohamed and A. Agarwal and D. Belgrave and K. Cho and A. Oh},
 pages = {26363--26374},
 publisher = {Curran Associates, Inc.},
 title = {IMED-RL: Regret optimal learning of ergodic Markov decision processes},
 url = {https://proceedings.neurips.cc/paper_files/paper/2022/file/a8c9f9ccc45771d2fd06bcd04ff3442e-Paper-Conference.pdf},
 volume = {35},
 year = {2022}
}

@article{GIVAN2003163,
title = {Equivalence notions and model minimization in Markov decision processes},
journal = {Artificial Intelligence},
volume = {147},
number = {1},
pages = {163-223},
year = {2003},
note = {Planning with Uncertainty and Incomplete Information},
issn = {0004-3702},
doi = {https://doi.org/10.1016/S0004-3702(02)00376-4},
url = {https://www.sciencedirect.com/science/article/pii/S0004370202003764},
author = {Robert Givan and Thomas Dean and Matthew Greig}
}

@article{gopalan2025dpo,
  title={Why DPO is a Misspecified Estimator and How to Fix It},
  author={Gopalan, Aditya and Chowdhury, Sayak Ray and Banerjee, Debangshu},
  journal={arXiv preprint arXiv:2510.20413},
  year={2025}
}

@misc{satheesh2026regretanalysisunichainaverage,
      title={Regret Analysis of Unichain Average Reward Constrained MDPs with General Parameterization}, 
      author={Anirudh Satheesh and Vaneet Aggarwal},
      year={2026},
      eprint={2602.08000},
      archivePrefix={arXiv},
      primaryClass={cs.LG},
      url={https://arxiv.org/abs/2602.08000}, 
}
